%% file: main.tex
\documentclass[11pt]{article}
\usepackage[margin=1in]{geometry}
\usepackage[expansion=false]{microtype}
\usepackage{amsmath,amssymb,amsthm}
\usepackage{booktabs,tabularx,array,makecell}
\usepackage{xcolor}
\usepackage{pgfplots}
\usepgfplotslibrary{groupplots}
\usetikzlibrary{arrows.meta,positioning,fit}
\pgfplotsset{compat=1.16}
\usepackage[numbers,sort&compress]{natbib}
\usepackage[colorlinks=true,linkcolor=blue!50!black,citecolor=blue!50!black,urlcolor=blue!50!black]{hyperref}

\input{numbers.tex}
\input{review_numbers.tex}
\input{review2_numbers.tex}

\input{review3_numbers.tex}
\input{review4_numbers.tex}

\newif\ifdraft \draftfalse
\newif\ifstagingcase \stagingcasetrue
\ifdraft
  
\else
  
\fi

\newtheorem{definition}{Definition}
\newtheorem{proposition}{Proposition}
\newcommand{\IES}{\mathrm{IES}}
\newcommand{\Agr}{A}
\newcommand{\one}{\mathbf{1}}
\newcommand{\pass}{\textsc{pass}}
\newcommand{\flag}{\textsc{flag}}
\newcommand{\abstain}{\textsc{abstain}}

\title{Does a Governed Confidence Score Know When a Text-to-SQL Answer Is Right?\\[4pt]
\large An exploratory study on BIRD and a pre-registered confirmation on Spider}
\author{Raghu Venkat \qquad Tricha Anjali\\
Numberz.ai Inc., Alpharetta, Georgia, USA\\
\texttt{\{raghu.venkat, tricha.anjali\}@numberz.ai}}
\date{October 2026 (version 2)}

\begin{document}
\maketitle

\begin{abstract}
Conversational analytics lets people ask questions of records spread across many systems and get
answers in plain language. It is useful only if each answer can be trusted and explained. A
production system of this kind attaches a confidence score, the Interpretive Evidence Score (IES),
to every answer it writes over a database, and uses it to show the answer (\pass), show it with a warning
(\flag) or withhold it (\abstain). The score's weights, caps and thresholds were set by design.
This paper is an audit of that design. On the 1{,}534 questions of BIRD dev, answered by GPT-5.1
with a benchmark-adapted production prompt and five samples per question, the designed score separates right
from wrong answers weakly (AUROC \birdExpDesignedAuroc, 95\% CI \birdExpDesignedAurocCI) and
withholds \birdExpDesignedCorrectWithheld{} of correct answers. Taking the score apart showed that
nearly all of its ranking signal comes from agreement among the samples, and that most of it is
lost where agreement is combined by a minimum with a checklist that drops to zero on any
unsupported claim, including claims in commentary. Before any Spider
answer existed we froze and hashed a rule in which agreement with the shown answer is the score
and deterministic checks act as vetoes, with deterministic and judge-based caps, together with its thresholds, truncation policy,
exclusion rule and four hypotheses; the answering harness and prompts were rebuilt afterwards,
and we list every deviation. On Spider dev the primary hypothesis held: AUROC rose from
\spiderDesignedAuroc{} to \spiderAgreeFirstAuroc{} (paired difference \spiderAgreeFirstDelta, 95\%
CI \spiderAgreeFirstDeltaCI), and the share of correct answers withheld fell from
\spiderDesignedCorrectWithheld{} to \spiderAgreeFirstCorrectWithheld, with \pass{} accuracy
\spiderAgreeFirstPassAcc{} against a base rate of \spiderAcc. The frozen rule was not the best
agreement-based rule on Spider: agreement alone (\spiderAgreeAuroc) and the rule with an
empty-result signal (\spiderVthreeAuroc) ranked higher in point estimate, though no difference
among the three is resolved by its interval. The critique judge already inside the pipeline, scored as a
standalone signal, ranked below agreement in every evaluation. A judge asked directly how likely
the answer is to be right, added after review and not pre-registered, was mixed on its own, but
its average with agreement ranked above agreement in every evaluation under model-judged labels
(\rjSpiderMean{} on Spider); the judge is the answering model, and the gain is smaller under
execution-match labels. One answer per Spider question from a second model family (Claude Sonnet
4.5) added ranking to agreement under both label definitions.
Under one result-matching policy the designed entropy and the share agreeing with the shown answer
rank answers almost equally well, so the gain comes from how agreement enters the score. The
checks always cost unqualified coverage, and whether they add ranking depends on how correctness
is defined. A pre-registered empty-result check did not meet its effect-size criterion. Agreement
cannot resolve a business definition that is absent from the evidence, which motivates checks
built from expert-approved definitions. Correctness here means the core answer to the question
asked. Labels are model-judged. In a pilot, two annotators working independently and blind to the
scores agreed with each other and with the model labels on all ten Spider answers they checked;
a full human audit has not been done, and all correctness-based results remain subject to it.
\end{abstract}

\section{Introduction}

\subsection{The problem: one question, many places}

A safety officer at a large operating organisation is asked how much of one chemical was used
last year, and at which sites. The question is short. The answer is spread across systems built
at different times for different purposes. Issue transactions record what was handed out. An
inventory system records what is on the shelf. Safety data sheets, often PDFs, record what each
product contains. Annual storage reports, built for a regulator, record what was held above a
threshold. A site table turns codes into names. Each system keys the chemical differently, by a
registry number, a trade name or a stock number; each has its own units and its own idea of a
year. Answering by hand means finding the right tables, reconciling their keys and units, and
checking the result with whoever knows which table means what. That can take weeks, and the
answer can still be wrong.

The same shape recurs in other organisations that have accumulated systems over time: a flying
unit that records the same sorties in flight files, workbooks, a database and briefing slides, or
maintenance knowledge split between PDF manuals, database work orders and free-text notes. The
question fits in a sentence; answering it means reconciling records never designed to be read
together. That work is often handed to a data team as a request for a script
(Figure~\ref{fig:changes}).

\begin{figure}[t]
\centering
\begin{tikzpicture}[x=1mm, y=1mm, font=\scriptsize,
  file/.style={draw, fill=gray!8, minimum width=10mm, minimum height=8mm, align=center, inner sep=1pt, anchor=south west},
  pcap/.style={text width=46mm, align=left, anchor=north west, font=\scriptsize, inner sep=0pt},
  ttl/.style={font=\scriptsize\bfseries, anchor=west, inner sep=0pt},
  arr/.style={-{Stealth[length=1.8mm]}, thick}]
\node[ttl] at (0,32) {1. The records};
\foreach \x/\lab in {0/{flight\\file}, 12/{work-\\book}, 24/{data-\\base}, 36/{report,\\slides}} \node[file] at (\x,12) {\lab};
\node[pcap] at (0,6) {The same events, written down in several kinds of file by different people at different times.};
\node[ttl] at (54,32) {2. Today};
\node[draw, fill=gray!8, align=center, text width=18mm, minimum height=8mm, anchor=south west] (sc) at (54,12) {a script for the data team};
\draw[arr, gray] (sc.east) -- ++(7,0) node[right, red!70!black, align=left] {days to\\weeks};
\node[pcap] at (54,6) {The question becomes a request for a script. The answer comes back later, if the script read the records the way the asker meant.};
\node[ttl] at (108,32) {3. Asked in plain words};
\node[draw=blue!60!black, rounded corners=2pt, fill=blue!5, align=left, text width=19mm, anchor=north west, inner sep=2pt] at (108,29) {``Which sorties went best?''};
\node[draw=green!45!black, rounded corners=2pt, fill=green!6, align=left, text width=21mm, anchor=north west, inner sep=2pt] at (132,29) {``By first detection, or longest hold?''};
\node[draw, fill=white, align=left, text width=44mm, anchor=north west, inner sep=2pt] at (108,17.5) {each number opens to its file\\ thin records are named};
\node[pcap] at (108,6) {The system asks back what a word means, answers from all the files at once, and says where the records are too thin.};
\end{tikzpicture}
\caption{What changes. The records and the people stay the same; the question stops needing a
script. Panel~3 is the requirement this paper is concerned with: every number traceable, and the
system willing to say when it cannot answer.}\label{fig:changes}
\end{figure}
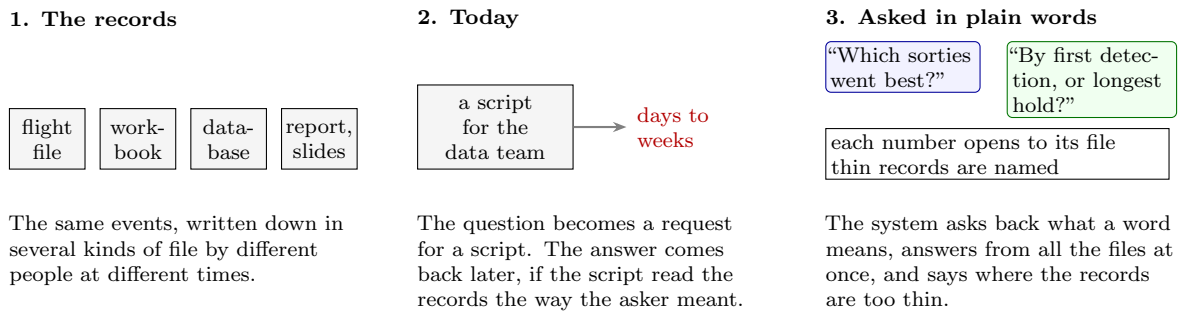

Conversational analytics can read more than databases. A document's tables can be read as rows,
a safety data sheet's composition section as one row per constituent, a log as timestamped
events, a scanned form through character recognition. What makes the answers combinable, and
checkable, is that every value ends up in the same form whatever file it came from: the value
itself, its uncertainty if any, the time it describes, whether it was measured, derived by a
named rule or asserted by a person or a document, where exactly it was found, and what it was
computed from (Figure~\ref{fig:record}). Once values are in that form, the question of whether an
answer is supported becomes a question about rows, tables and joins, which is the setting this
paper measures.

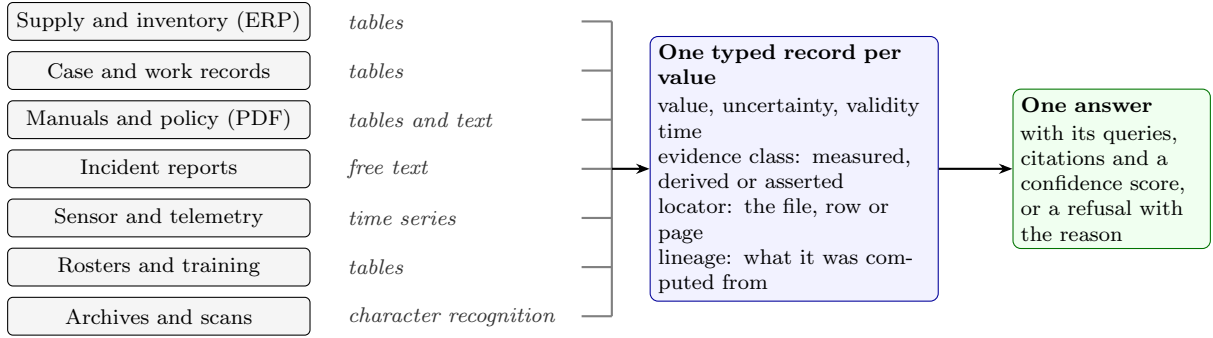
\begin{figure}[t]
\centering
\begin{tikzpicture}[x=1mm, y=1mm, font=\scriptsize,
  src/.style={draw, rounded corners=2pt, fill=gray!8, minimum width=40mm, minimum height=5mm, align=left, anchor=west, inner sep=2pt},
  rd/.style={font=\scriptsize\itshape, text=gray!40!black, anchor=west, inner sep=0pt},
  arr/.style={-{Stealth[length=1.8mm]}, thick}]
\foreach \i/\src/\rdr in {0/Supply and inventory (ERP)/tables, 1/Case and work records/tables, 2/Manuals and policy (PDF)/tables and text, 3/Incident reports/free text, 4/Sensor and telemetry/time series, 5/Rosters and training/tables, 6/Archives and scans/character recognition} {
  \node[src] at (0,-\i*6.5) {\src};
  \node[rd] at (45,-\i*6.5) {\rdr};
  \draw[thick, gray] (76,-\i*6.5) -- (80,-\i*6.5);
}
\draw[thick, gray] (80,0) -- (80,-39);
\node[draw=blue!60!black, rounded corners=3pt, fill=blue!5, text width=36mm, align=left, inner sep=3pt, anchor=west] (rec) at (85,-19.5)
 {\textbf{One typed record per value}\\[1pt]
  value, uncertainty, validity time\\
  evidence class: measured, derived or asserted\\
  locator: the file, row or page\\
  lineage: what it was computed from};
\draw[arr] (80,-19.5) -- (rec.west);
\node[draw=green!45!black, rounded corners=3pt, fill=green!6, text width=24mm, align=left, inner sep=3pt, anchor=west] (ans) at (133,-19.5)
 {\textbf{One answer}\\[1pt] with its queries, citations and a confidence score, or a refusal with the reason};
\draw[arr] (rec.east) -- (ans.west);
\end{tikzpicture}
\caption{Many kinds of record, one form. Each source is read by the reader suited to it, and every
value is held with its time, its evidence class, where it was found and what it was computed from.
Disagreements between sources are settled by a precedence the organisation declares, with the losing
value kept visible.}\label{fig:record}
\end{figure}

A smaller example shows how a single table can mislead. Asked how many sites in one state hold a chemical, a reasonable query
counts the rows whose \texttt{location} column names the state. It returns 2. The true answer is
213: the organisation stopped filling the location column years ago, and the state now lives in
the prefix of the site identifier. Nothing errored, the column exists, and the count is
arithmetically correct. The schema did not mean what its names implied, and the knowledge needed
to see that sat with the people who maintain the data. (The example is illustrative, in the
pattern of cases met in practice.)

\subsection{Conversational analytics helps, if it can be trusted}

Language models change the cost of that work. Given the schema, a model can find candidate
tables, write and run queries, read a document's composition section as rows, and explain its
result in plain language, often within one session. Systems that do this, often called conversational
analytics, are useful precisely where records are scattered and the people who need answers are
not the people who write SQL.

Fluency is not the hard part. The failure that matters is the confident wrong answer: a valid
query over real tables that answers a different question from the one asked, written up in the
same assured prose as a right one. The model rarely invents the number; it chooses the wrong
relationship, the wrong column or the wrong definition, and the arithmetic on top is correct. A
reader checking by eye may not see it. For such a system to be used for decisions, two things
matter more than the answer's polish. Each answer has to show its work, the queries it ran, the
records it read and the definitions it used, so that a reviewer can check it. And the system has
to decide, answer by answer, whether the evidence supports showing the result at all, and say
why when it does not.

Figure~\ref{fig:ov1} shows the arrangement this paper assumes. Records from several sources and
the organisation's recorded definitions feed a model that writes queries; the model is sampled
several times; a set of checks outside the model grades the answer and its trace; and a gate
shows the answer, shows it with a warning, or withholds it, naming the reason. This paper is about
the last two boxes.

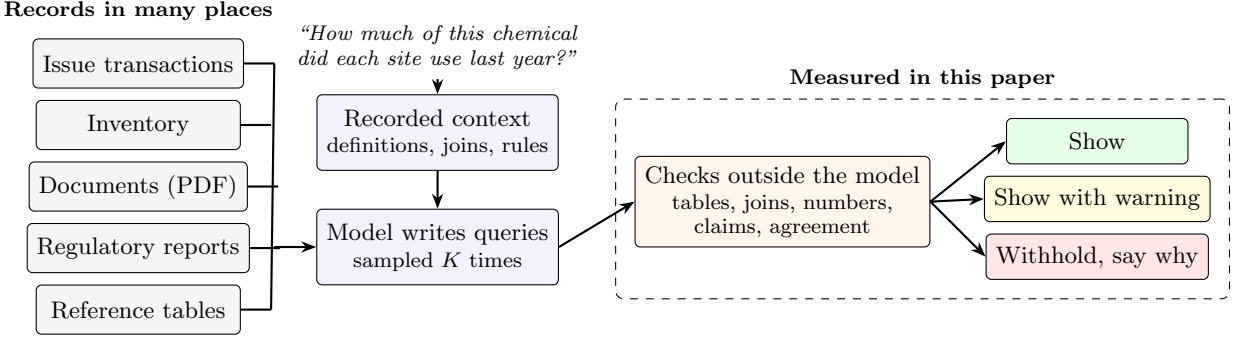
\begin{figure}[t]
\centering
\begin{tikzpicture}[font=\footnotesize, node distance=4mm,
  src/.style={draw, rounded corners=2pt, fill=gray!8, minimum width=27mm, minimum height=6.5mm, align=center},
  box/.style={draw, rounded corners=2pt, fill=blue!5, minimum width=30mm, minimum height=10mm, align=center},
  chk/.style={draw, rounded corners=2pt, fill=orange!8, minimum width=30mm, minimum height=10mm, align=center},
  gt/.style={draw, rounded corners=2pt, minimum width=24mm, minimum height=6mm, align=center},
  arr/.style={-{Stealth[length=2mm]}, thick}]
\node[src] (s1) {Issue transactions};
\node[src, below=1.5mm of s1] (s2) {Inventory};
\node[src, below=1.5mm of s2] (s3) {Documents (PDF)};
\node[src, below=1.5mm of s3] (s4) {Regulatory reports};
\node[src, below=1.5mm of s4] (s5) {Reference tables};
\node[font=\scriptsize\bfseries, above=1mm of s1] {Records in many places};
\node[box, right=10mm of s2, yshift=-1mm] (ctx) {Recorded context\\[-1pt]{\scriptsize definitions, joins, rules}};
\node[box, below=5mm of ctx] (llm) {Model writes queries\\[-1pt]{\scriptsize sampled $K$ times}};
\node[font=\scriptsize, above=2mm of ctx, align=center] (q) {\emph{``How much of this chemical}\\\emph{did each site use last year?''}};
\node[chk, right=10mm of llm, yshift=6mm] (chk) {Checks outside the model\\[-1pt]{\scriptsize tables, joins, numbers,}\\[-2pt]{\scriptsize claims, agreement}};
\node[gt, right=10mm of chk, yshift=8mm, fill=green!10] (p) {Show};
\node[gt, below=1.5mm of p, fill=yellow!15] (f) {Show with warning};
\node[gt, below=1.5mm of f, fill=red!10] (a) {Withhold, say why};
\foreach \n in {s1,s2,s3,s4,s5} \draw[thick] (\n.east) -- ++(4mm,0);
\draw[thick] ([xshift=4mm]s1.east) -- ([xshift=4mm]s5.east);
\draw[arr] ([xshift=4mm]s1.east |- llm.west) -- (llm.west);
\draw[arr] (q) -- (ctx);
\draw[arr] (ctx) -- (llm);
\draw[arr] (llm.east) -- (chk.west);
\draw[arr] (chk.east) -- (p.west);
\draw[arr] (chk.east) -- (f.west);
\draw[arr] (chk.east) -- (a.west);
\node[draw, dashed, rounded corners=3pt, inner sep=2.5mm, fit=(chk)(p)(f)(a), label={[font=\scriptsize\bfseries]above:Measured in this paper}] {};
\end{tikzpicture}
\caption{The setting. Records from several systems and the owners' recorded definitions feed a
model that writes queries. Checks outside the model grade the answer and its trace; a gate shows
the answer, shows it with a warning, or withholds it with a reason. This paper measures the checks
and the gate.}\label{fig:ov1}
\end{figure}

\subsection{What can go wrong, and what was designed to catch it}

The failures follow from the problem. A query can name a table that does not exist, or join two
columns whose keys barely overlap. A column can hold a stand-in code such as $-1$ that a sum
treats as a value. The written answer can quote a number that no returned row contains, or add
claims the data does not support. The model can read the question one way on one run and another
way on the next. And it can read it the same wrong way every time, as in both examples above.

The score studied here was designed against that list. After an answer is written, a checker goes
back over the trace and grades whether the queries used real tables and sound joins, whether the
numbers in the text match returned rows, whether the model claimed things the data does not
support, whether every part of the question was answered, and whether independently sampled
answers agree. A few failures are treated as too serious to average away and force the answer
below the abstain line. The design was set by judgement, with the stated intent that its weights
and thresholds would be calibrated on held-out questions. It had not been measured.

\subsection{This study}

This paper audits that design and measures its discrimination, calibration and gate behaviour on
public text-to-SQL benchmarks, where every answer can be checked against a reference. Its scope
is narrower than the setting above: it evaluates a reference implementation of the scoring and
gating layer on structured databases, and does not evaluate document extraction, source
reconciliation, or context-aware answer generation. It asks three questions.

\begin{description}
\item[E1.] Does the score predict correctness? Are \pass{} answers mostly right and withheld
answers mostly wrong?
\item[E2.] Does each check catch the failure it is built for, when that failure is planted in an
answer that was right?
\item[E3.] Does each piece earn its place, measured by removing or changing one piece at a time?
\end{description}

We answered these on BIRD dev and the answers were not flattering to the design. The useful
signal came from one place, agreement among sampled answers, and the score combined it in a way
that discarded much of the information. We then wrote down a revised rule, froze and hashed the code,
wrote the test and its pass criteria down, and ran it on a second benchmark, Spider dev, that
played no part in choosing the rule. We close by returning to the problem above: a checker cannot
recover an intended business definition that is absent from the evidence it can see, and those
definitions live with the people who own the data.

\paragraph{Contributions.}
\begin{itemize}
\item An empirical audit of a designed, governed answer score on public text-to-SQL data with the
production prompt and a current model, reported as measured, including the parts that did not
work.
\item A decomposition of where the score's signal comes from and what each part costs, including
an ablation that separates the agreement statistic from the result-matching policy, and a
measurement of what the deterministic and judge checks cost in coverage on top of agreement.
\item A confirmation on Spider in which the rule code, thresholds and analysis were fixed and
hashed before any Spider answer existed, with every hypothesis reported pass or fail and every
deviation from the registration listed, and a replication on BIRD under the same harness.
\item A labelling protocol for answers written in an application's own style, and a measurement
of how much the labeller moves the results.
\end{itemize}
The agreement statistic itself is not new: execution-based consistency across sampled queries
is an established confidence signal for text-to-SQL (Section~\ref{sec:related}). What this paper
adds is the audit of a designed score built around it, tested on data the rule was not
chosen on.

\section{Related work}\label{sec:related}

\paragraph{Consistency as confidence.} Sampling several outputs and measuring their agreement is
a standard black-box confidence signal: self-consistency for reasoning \citep{wang2023selfconsistency},
SelfCheckGPT for hallucination \citep{manakul2023selfcheckgpt}, and semantic entropy, which
clusters samples by meaning before measuring their spread \citep{kuhn2023semantic,farquhar2024semantic}.
The designed $L_2$ term is semantic entropy with meaning fixed by execution.

\paragraph{Confidence for text-to-SQL.} \citet{maleki2026confidence} compare black-box and
white-box confidence estimates for text-to-SQL and find consistency-based methods strongest among
black-box approaches, with query execution a useful additional signal; the share of sampled
queries that return the same result is the agreement statistic used here. \citet{richardson2026predicts}
compare many signals on BIRD and Spider against execution-match labels and report AUROC of about
0.61 to 0.68 for black-box signals including execution self-consistency, 0.72 to 0.78 for an
LLM judge, and about 0.82 for an ensemble of judges from two providers. Their labels, models and
prompts differ from ours, so the numbers are not directly comparable. Their result that
verification signals can outrank consistency is relevant here. We score the critique judge that
already runs inside the pipeline as a standalone signal and, added after review, a judge asked
directly for the probability that an answer is right, which is the kind they report
(Section~\ref{sec:ablate}).
\citet{somov2025confidence} attach entropy-based selective classifiers and calibration to
text-to-SQL models and find an encoder-decoder model better calibrated than large general models.
\citet{chen2025rts} build abstention into the text-to-SQL pipeline itself, with conformal
guarantees at schema linking and a human in the loop; our gate acts on the finished answer.

\paragraph{Selective prediction and calibration.} Answering or abstaining by a score is selective
classification \citep{elyaniv2010foundations,geifman2017selective}; under domain shift it was
studied for question answering by \citet{kamath2020selective}. We report the gate's coverage and
accuracy beside AUROC, and calibration error as defined by \citet{guo2017calibration}.

\paragraph{Pre-registration.} Fixing hypotheses and analysis before seeing confirmatory data
\citep{nosek2018preregistration} is uncommon in this literature; we adopt it because the rules
under test were chosen on the exploratory data.

\section{The score}\label{sec:score}

This section defines the designed score and the three frozen rules precisely enough to be
re-implemented. The reference code is \texttt{score\_v2.py} and \texttt{agreement.py} in the
released bundle.

\subsection{Setting and notation}

A question $q$ is asked over a database $D$. The answering system is sampled $K=5$ times
independently. Sample $k$ produces a tool trace, a written answer, and a final query result
$R_k$, a multiset of rows, or $R_k=\bot$ if it ran no successful query. Sample $0$ is the answer
shown to the user and scored; samples $1,\dots,4$ are used only for agreement. A label
$y\in\{0,1\}$ says whether the scored answer is right (Section~\ref{sec:labels}).

\subsection{The designed score (bench-1.0)}

The designed score is a weighted mean of five terms over those that can be measured, followed
by caps and a three-way gate. The weights are 0.30 for GROUNDED, 0.20 each for SCHEMA, SAFETY
and RETRIEVE, and 0.10 for COMPLETE. With $M$ the set of measured terms $T_j$,
\begin{equation}
\IES_0=\frac{\sum_{j\in M}w_jT_j}{\sum_{j\in M}w_j},\qquad
\IES=\begin{cases}\min(\IES_0,\,0.59)&\text{if any hard cap fires,}\\
\min(\IES_0,\,0.79)&\text{if evidence is thin,}\\ \IES_0&\text{otherwise,}\end{cases}
\end{equation}
and the gate is \pass{} if $\IES\ge0.80$, \flag{} if $0.60\le\IES<0.80$, \abstain{} otherwise.
Evidence is thin when fewer than three terms are measured or GROUNDED is not.

\paragraph{GROUNDED} is the mean of the defined parts among three shares checked against a
profile of $D$: $g_{\mathrm{tab}}$, the share of parsed queries whose tables all exist (CTE names
and the SQLite catalogue are not tables); $g_{\mathrm{join}}$, the share of equi-join column
pairs $(a,b)$ with key overlap
\begin{equation}
\rho(a,b)=\max\!\left(\frac{|K_a\cap K_b|}{|K_a|},\,\frac{|K_a\cap K_b|}{|K_b|}\right)\ge0.85,
\end{equation}
where $K_x$ is the set of distinct non-null values of column $x$; and $g_{\mathrm{val}}$: among
returned columns whose source column carries a stand-in code (a value such as $-1$ or $9999$
covering at least 0.5\% of an otherwise non-negative or much smaller column), the share whose
returned values contain no stand-in. Like the other two shares, larger is better. A missing
table, a join with $\rho<0.85$, or a stand-in returned as a value fires a hard cap.

\paragraph{SCHEMA} checks the numbers in the written answer. A number $x$ written with $d$
decimals in the Executive Summary or Detailed Analysis sections is \emph{cited} if some
returned cell, cell$\times100$, or row count $c$ satisfies $|x-c|\le\tfrac12 10^{-d}$; numbers
that appear in the question or hint, and the integers 1 to 5, are ignored. SCHEMA is 1 if at
least one number is cited and unmeasured otherwise; any number neither cited nor ignored fires
the hard cap \textsc{uncited\_number}. SCHEMA is therefore a constant 1 whenever it is measured, and
numerical verification acts only through the cap; the ``drop SCHEMA'' ablation in
Table~\ref{tab:e3} removes that constant from the weighted mean (and changes when evidence counts as
thin), and does not test removing the number check itself.

\paragraph{COMPLETE} is the share of the question's parts that the answer covers, as listed by
a judge model that never sees the reference answer. \textbf{RETRIEVE} is not measured on these
benchmarks, since every value is a database cell, and the weights renormalise over the rest.

\paragraph{SAFETY} is the hallucination check TDD-H, the minimum of two layers here (a third,
a competence boundary, is not run):
\begin{equation}
\mathrm{SAFETY}=\min(L_1,L_2),\qquad L_1=1-r_1,\qquad
L_2=1-\frac{H(\pi)}{\log K}.
\end{equation}
$L_1$ is a checklist: $r_1=1$ if any deterministic tag fires or the judge reports an unsupported
claim, $r_1=\tfrac12$ if any other judge or heuristic tag fires, and $r_1=0$ otherwise. An
unsupported claim of any kind also fires a hard cap. $L_2$ is scaled semantic entropy
\citep{kuhn2023semantic,farquhar2024semantic} with meaning fixed by execution: samples are
clustered by the signature of their result set (full rows, floats rounded to six places), a
sample with no successful query is its own cluster, and $\pi$ is the distribution of the $K$
samples over clusters.

\subsection{What the entropy term cannot see}

\begin{proposition}\label{prop:entropy}
$L_2$ is invariant under any permutation of the $K$ samples. Hence it takes the same value
whether the scored sample $0$ lies in the largest cluster or alone.
\end{proposition}
\begin{proof}
$\pi$ is a function of the multiset of cluster sizes, which does not depend on sample order.
\end{proof}

For $K=5$, if four samples agree and one differs, $\pi=(0.8,0.2)$ and $L_2=1-H(\pi)/\log5\approx
0.69$ in both cases: when the scored answer is with the four, and when it is the one. The same
holds for any split. So $L_2$ measures how settled the model is in general, not whether the
answer being shown is the one it settles on. The quantity the gate needs is the second. How often
the difference matters in practice is measured in Section~\ref{sec:ablate}; it turns out to be
small.

\begin{definition}[Agreement with the scored answer]
\begin{equation}
\Agr=\frac1K\sum_{k=0}^{K-1}\one\!\left[R_k\simeq R_0\right],\qquad \Agr=0\ \text{if}\ R_0=\bot,
\end{equation}
where $R\simeq R'$ holds when both have the same row count and the narrower result equals, as a
multiset of rows, a projection of the wider one onto an injective choice of its columns (extra
columns allowed), after normalising values (floats to four places, numeric strings to numbers).
\end{definition}

$\Agr$ takes values in $\{0,0.2,\dots,1\}$ and counts the scored sample itself, so
$\Agr\ge0.2$ whenever $R_0\neq\bot$. Results are stored up to 100 rows; when either result is
truncated and the row counts are equal, $\one[\cdot]$ is undetermined. The pre-registered
policy counts such pairs as agreeing and reports the bound with them counted as disagreeing.

\subsection{The frozen rules}

Three rules were written from the BIRD findings and frozen at commit \texttt{cb66ea7} (code bundle
\texttt{6e3e723e8f76711d}) on 29 September 2026. None has a fitted parameter; all keep the
designed thresholds.

\paragraph{v2 rule (2.0).} The designed score with three changes: $L_2$ is replaced by $\Agr$;
$L_1$ and the unsupported-claim cap count only judge issues that the judge classes as concerning
the answer, not the commentary around it; and an answer with COMPLETE $<0.5$ is capped at 0.79.

\paragraph{v2 agreement-first (2.1), the primary rule.} The score is $\Agr$ itself, with
deterministic checks as vetoes and a mix of deterministic and judge-based checks as caps:
\begin{equation}
s=\begin{cases}
\min(\Agr,0.59)&\text{if a veto fires,}\\
\min(\Agr,0.79)&\text{if a cap fires,}\\
\Agr&\text{otherwise,}\end{cases}
\end{equation}
where the vetoes are a missing table, a bad join and a stand-in returned as a value, and the caps
are an uncited number and thin evidence (deterministic) and a judge issue about the answer and
COMPLETE $<0.5$ (both from the GPT-5 mini judge).
With $K=5$ the gate reads directly: \pass{} needs at least four of five samples to agree and no
check to fire; three agreeing samples give \flag; two or fewer give \abstain. Every \flag{} and
\abstain{} carries the names of the checks that caused it.

\paragraph{v3 (3.0).} Agreement-first plus zero-cost signals aimed at confident misreadings,
where all samples agree and are wrong. One firing signal caps at 0.79, two or more at 0.59. Only
the empty-result signal (the final query returned no rows) was frozen for test.

\section{Method}\label{sec:method}

\subsection{Data}

BIRD dev \citep{li2023bird}: 1{,}534 questions over 11 SQLite databases, each with reference
SQL and an evidence hint that we pass to the model as a hint. Spider dev \citep{yu2018spider}:
1{,}034 questions over 20 databases, with reference SQL and no hint.

\subsection{System under test}

The answering model is GPT-5.1 with the system's production system prompt, trimmed only of
sections for tools the benchmark does not provide (vector search, charts, documents and
citations, workflows, forecasting, a domain table guide); whole sections and lines were removed
and nothing reworded (Appendix~\ref{app:prompt}). The model has three tools over the question's
database: list tables, table details with sample rows, and read-only SQL that shows it up to 100
rows. One added system message states that the last query must return exactly the requested
columns and rows, since both benchmarks grade result sets. Five answers are sampled per question
at default decoding on OpenAI's flex tier. The IES judge is GPT-5 mini and never sees the
reference.

\subsection{Correctness labels}\label{sec:labels}

The benchmarks' own rule, identical result sets, marks \birdExpExecMatch{} of BIRD answers
correct. It fails in two directions for an application that answers in its own style. It marks
right answers wrong: names instead of ids, 0.83 instead of 83\%, both values when the question
asks which is larger. And it trusts reference SQL that is sometimes wrong, such as a formula
that contradicts the hint.

We therefore label each answer with GPT-5.4 grading against the reference. The labeller sees
the question, hint, reference SQL and result, the answer's SQL and result, mechanical
comparisons of the two, and the written answer, and rules on the answer (correct, incorrect,
unclear), on the difference (none, form only, content), and on the reference (ok, wrong,
question ambiguous). Answers ruled unclear or question-ambiguous are left out. On the 29
September BIRD run the labeller judged \birdExpRefWrong{} reference answers wrong
(\birdExpRefWrongPct) and kept \birdExpN{} answers, \birdExpWrong{} of them wrong.

The labeller is a model and its rulings move the results. Section~\ref{sec:labeller} measures
how much, by relabelling the same answers with the rebuilt labeller, and Section~\ref{sec:labeller}
reports a ten-answer human pilot. Correctness here means that the answer
to the question asked is right: the returned rows and the figures the written answer reports
from them. Unsupported commentary around a correct answer is a real defect, but it is recorded
separately and does not make the answer incorrect. The score is judged against the same target.

\subsection{Metrics}

For a score $s$ and labels $y$ over $n$ answers we report AUROC with a 95\% interval from 2{,}000
bootstrap resamples of answers \citep{efron1993bootstrap}, the paired bootstrap interval of
$\mathrm{AUROC}(s')-\mathrm{AUROC}(s)$ on the same resamples, expected calibration error over
10 equal-width bins \citep{guo2017calibration}, the Brier score, and three gate quantities:
\begin{align*}
\text{\pass{} accuracy}&=P(y=1\mid\pass),\\
\text{correct withheld}&=P(\abstain\mid y=1),\\
\text{wrong withheld}&=P(\abstain\mid y=0).
\end{align*}
Answers to questions on the same database share a schema and its conventions, so their errors are
not independent; resampling answers can understate the width of an interval. For the headline
differences we therefore also report a database-level cluster bootstrap, resampling whole
databases (11 in BIRD, 20 in Spider) with every answer on them (Table~\ref{tab:cluster}). With so
few clusters these intervals are themselves approximate. Ties in a discrete score count one half in AUROC. The fitted variant reported as exploratory uses
logistic regression with whole-database cross-fitting over 20 random splits, so no database is
on both sides of a fit.

\subsection{Order of work}

\begin{enumerate}
\item \textbf{Exploratory, BIRD, 29 September.} E1 to E3 on one full run. The frozen rules come
from this step and are evaluated on the same answers, so their BIRD numbers are in-sample.
\item \textbf{Freeze and pre-registration, 29 September.} Code bundle hashed; hypotheses H1 to
H4 and their pass criteria written before any Spider answer existed (Section~\ref{sec:prereg}).
\item \textbf{Harness rebuild, 11 October.} The benchmark harness used on 29 September was not
recoverable. We rebuilt it and validated the rebuilt designed scorer field for field against all
1{,}534 stored 29 September records before using it.
\item \textbf{Confirmatory, Spider, 11--12 October,} on the rebuilt harness, evaluated by the
frozen bundle.
\item \textbf{Replication, BIRD, 11 October,} a full rerun on the rebuilt harness with the same
settings as Spider. This is not pre-registered and is reported as a check that the harness
rebuild did not change the BIRD picture.
\end{enumerate}

\section{Exploratory results on BIRD}\label{sec:explore}

\subsection{E1: the designed score}

GPT-5.1 with the production prompt answered \birdExpAcc{} of the \birdExpN{} labelled BIRD
questions correctly. Table~\ref{tab:bird} gives the designed score and the rules that came out
of this analysis, on the same answers.

\begin{table}[t]
\centering\footnotesize\setlength{\tabcolsep}{4pt}
\caption{BIRD dev, 29 September run, \birdExpN{} labelled answers. The rules below the line were
derived from these answers and are in-sample here.}\label{tab:bird}
\begin{tabular}{lccccc}
\toprule
Score & AUROC (95\% CI) & \makecell{\pass\\accuracy} & \makecell{\pass\\share} & \makecell{Correct\\withheld} & \makecell{Wrong\\withheld}\\
\midrule
Designed & \birdExpDesignedAuroc{} (\birdExpDesignedAurocCI) & \birdExpDesignedPassAcc & \birdExpDesignedPassShare & \birdExpDesignedCorrectWithheld & \birdExpDesignedWrongWithheld\\
\midrule
v2 rule & \birdExpRuleAuroc{} (\birdExpRuleAurocCI) & \birdExpRulePassAcc & \birdExpRulePassShare & \birdExpRuleCorrectWithheld & \birdExpRuleWrongWithheld\\
Agreement alone & \birdExpAgreeAuroc{} (\birdExpAgreeAurocCI) & \birdExpAgreePassAcc & \birdExpAgreePassShare & \birdExpAgreeCorrectWithheld & \birdExpAgreeWrongWithheld\\
Agreement-first & \birdExpAgreeFirstAuroc{} (\birdExpAgreeFirstAurocCI) & \birdExpAgreeFirstPassAcc & \birdExpAgreeFirstPassShare & \birdExpAgreeFirstCorrectWithheld & \birdExpAgreeFirstWrongWithheld\\
v3 (empty result) & \birdExpVthreeAuroc{} (\birdExpVthreeAurocCI) & \birdExpVthreePassAcc & \birdExpVthreePassShare & \birdExpVthreeCorrectWithheld & \birdExpVthreeWrongWithheld\\
Fitted (exploratory) & \birdExpFittedAuroc{} (\birdExpFittedAurocCI) & \birdExpFittedPassAcc & \birdExpFittedPassShare & \birdExpFittedCorrectWithheld & \birdExpFittedWrongWithheld\\
\bottomrule
\end{tabular}
\end{table}

The \pass{} gate holds up at the top: \pass{} answers are right \birdExpDesignedPassAcc{} of
the time against \birdExpAcc{} overall. Below the top the score carries little information.
It withholds \birdExpDesignedCorrectWithheld{} of correct answers and
\birdExpDesignedWrongWithheld{} of wrong ones. The hard cap puts \exAtCap{} answers at
exactly 0.59, and \exAtCapRight{} of them are correct, close to the average. The score is not a
probability: answers scoring 0.9--1.0 are right \exBandNine{} of the time, 0.8--0.9
\exBandEight, and 0.7--0.8 \exBandSeven.

\subsection{E3: what each piece contributes}

Recomputing the score from saved components, one change at a time, with no new model calls:

\begin{table}[t]
\centering\footnotesize\setlength{\tabcolsep}{4pt}
\caption{Single changes to the designed score, recomputed from saved components (BIRD, 29 September run).}\label{tab:e3}
\begin{tabular}{lccc}
\toprule
Change & AUROC & \pass{} acc. & \pass{} share\\
\midrule
None (designed) & \birdExpDesignedAuroc & \birdExpDesignedPassAcc & \birdExpDesignedPassShare\\
Drop GROUNDED & \exEDropGroundedAuroc & \exEDropGroundedPassAcc & \exEDropGroundedPassShare\\
Drop SCHEMA & \exEDropSchemaAuroc & \exEDropSchemaPassAcc & \exEDropSchemaPassShare\\
Drop COMPLETE & \exEDropCompleteAuroc & \exEDropCompletePassAcc & \exEDropCompletePassShare\\
Drop SAFETY & \exEDropSafetyAuroc & \exEDropSafetyPassAcc & \exEDropSafetyPassShare\\
SAFETY without $L_2$ & \exESafetyNoAuroc & \exESafetyNoPassAcc & \exESafetyNoPassShare\\
SAFETY as mean instead of min & \exESafetyMeanAuroc & \exESafetyMeanPassAcc & \exESafetyMeanPassShare\\
No thin-evidence cap & \exENoThinCapAuroc & \exENoThinCapPassAcc & \exENoThinCapPassShare\\
No hard caps & \exENoHardCapsAuroc & \exENoHardCapsPassAcc & \exENoHardCapsPassShare\\
Unsupported claims capped only when about the answer & \exEUnsupportedAnswerOnlyAuroc & \exEUnsupportedAnswerOnlyPassAcc & \exEUnsupportedAnswerOnlyPassShare\\
Unsupported claims ignored & \exEUnsupportedIgnoredAuroc & \exEUnsupportedIgnoredPassAcc & \exEUnsupportedIgnoredPassShare\\
Agreement alone, relaxed rows & \birdExpAgreeAuroc & \birdExpAgreePassAcc & \birdExpAgreePassShare\\
\bottomrule
\end{tabular}
\end{table}

Three findings shaped the frozen rules.

\emph{Agreement carries the signal.} Removing $L_2$ from SAFETY drops AUROC to
\exESafetyNoAuroc; agreement alone reaches \birdExpAgreeAuroc. Dropping GROUNDED, SCHEMA (a constant when measured)
or COMPLETE, removing the thin-evidence cap, or taking the mean instead of the minimum each move
AUROC by \exMaxMinorDelta{} or less. On BIRD the databases are clean and the model rarely invents tables, so
the grounding terms are almost always 1 and carry no ranking information here.

\emph{Unsupported-claim handling is the largest observed source of degradation.} It acts twice:
through the checklist $L_1$, which drops to zero on any unsupported claim and so pulls SAFETY to
zero through the minimum, and through the hard cap. Section~\ref{sec:ablate} separates the two;
most of the loss comes through $L_1$ and the minimum. The cap fires on \exFireCapUnsupportedAnyShare{} of
answers, and those answers are right \exFireCapUnsupportedAnyRight{} of the time, close to the
average. Most of the answers it caps (\exFireUnsupportedCommentaryN{} of
\exFireCapUnsupportedAnyN) carry a claim the judge itself classed as commentary, and answers with
such a claim are right \exFireUnsupportedCommentaryRight{} of the time.
The claims are real (``Super Strength appears frequently in the dataset'', from a single count),
and they come from the production prompt's instruction to enrich every answer with comparisons
and interpretation. They are a fault in what the user reads, but they say little about whether
the answer is right.

\emph{The checks that work are rare.} When the judge reports an intent mismatch the answer is
right only \exFireJudgeIntentMismatchRight{} of the time; after a bad join,
\exFireCapBadJoinRight; after a wrong formula, \exFireJudgeWrongFormulaRight. Each of these fires
on \exFireRareMax{} of answers or fewer.

\subsection{E2: planted errors}

From answers labelled correct we planted each of eight checklist errors, 60 per kind
(Appendix~\ref{app:inject}). For errors that change the query, the changed SQL was executed and
the model rewrote its answer from the new rows, so the text agrees with the broken result. All
\exControlsOk{} of \exPlanted{} controls, rescored unchanged, reproduced their stored scores.

\begin{table}[t]
\centering\footnotesize\setlength{\tabcolsep}{4pt}
\caption{Planted errors, 60 per kind. Withheld, not \pass{} and own check are shares of the 60
altered answers. Control withheld is the share of the 60 unchanged answers the designed score
already withheld; newly withheld counts, among the controls it did not withhold, those the
planted error moved to \abstain.}\label{tab:e2}
\begin{tabular}{lccccc}
\toprule
Error & Withheld & Not \pass & Own check & \makecell{Control\\withheld} & \makecell{Newly\\withheld}\\
\midrule
Fake table & \exInjFakeTableWithheld & \exInjFakeTableNotPass & \exInjFakeTableOwnCheck & \rtInjFakeTableCtrlWithheld & \rtInjFakeTableNewWithheldN\\
Bad join & \exInjBadJoinWithheld & \exInjBadJoinNotPass & \exInjBadJoinOwnCheck & \rtInjBadJoinCtrlWithheld & \rtInjBadJoinNewWithheldN\\
Stand-in as value & \exInjStandInWithheld & \exInjStandInNotPass & \exInjStandInOwnCheck & \rtInjStandInCtrlWithheld & \rtInjStandInNewWithheldN\\
Wrong number & \exInjWrongNumberWithheld & \exInjWrongNumberNotPass & \exInjWrongNumberOwnCheck & \rtInjWrongNumberCtrlWithheld & \rtInjWrongNumberNewWithheldN\\
Wrong time window & \exInjTimeWindowWithheld & \exInjTimeWindowNotPass & \exInjTimeWindowOwnCheck & \rtInjTimeWindowCtrlWithheld & \rtInjTimeWindowNewWithheldN\\
Wrong aggregate & \exInjWrongAggregateWithheld & \exInjWrongAggregateNotPass & \exInjWrongAggregateOwnCheck & \rtInjWrongAggregateCtrlWithheld & \rtInjWrongAggregateNewWithheldN\\
Intent flip & \exInjIntentFlipWithheld & \exInjIntentFlipNotPass & \exInjIntentFlipOwnCheck & \rtInjIntentFlipCtrlWithheld & \rtInjIntentFlipNewWithheldN\\
Partial answer & \exInjPartialAnswerWithheld & \exInjPartialAnswerNotPass & \exInjPartialAnswerOwnCheck & \rtInjPartialAnswerCtrlWithheld & \rtInjPartialAnswerNewWithheldN\\
\bottomrule
\end{tabular}
\end{table}

The four deterministic checks withheld every planted error of their kind, each by its own check.
The judge-based checks withheld \exInjJudgeRange. About a third to two-fifths of the unchanged
controls were already withheld, so the withheld column overstates what a planted error adds;
among controls the designed score did not withhold, every deterministic error was newly withheld,
and the judge-based errors were newly withheld \rtInjWrongAggregateNewWithheld{} (wrong
aggregate) to \rtInjTimeWindowNewWithheld{} (time window) of the time. A judge check that fires
does not always stop a \pass: a non-commentary judge issue sets $r_1=\tfrac12$ without a hard cap,
and with the other terms at 1 the weighted mean can still reach the \pass{} line. Of the
\rtInjWrongAggregateFired{} planted wrong aggregates where the judge fired, \rtInjWrongAggregateFiredButPass{}
still passed. Real errors are harder: the designed score
withholds \birdExpDesignedWrongWithheld{} of the wrong answers the model actually produced.
Fault injection overstates real-world detection and should be reported beside it.

\section{Pre-registered confirmation on Spider}\label{sec:prereg}

\subsection{What was fixed in advance}

On 29 September, before any Spider answer was generated, we fixed the scorer code (bundle
\texttt{6e3e723e8f76711d}; \texttt{evaluate.py} prints the hash and a mismatch voids the run),
the rules under test (agreement-first as primary, v3 with the empty-result signal), the
thresholds, the truncation policy, the exclusion rule for labels, and four hypotheses
\citep{nosek2018preregistration}:

\begin{description}
\item[H1 (primary).] Agreement-first ranks answers better than the designed score: the paired
bootstrap 95\% interval of $\mathrm{AUROC}(\text{agree-first})-\mathrm{AUROC}(\text{designed})$
lies above 0.
\item[H2.] Agreement-first withholds fewer correct answers than the designed score, and its
\pass{} accuracy is not lower by more than 2 points.
\item[H3.] An empty final result marks a likely wrong answer: accuracy when it fires is at least
30 points below overall accuracy, tested only if it fires on 10 or more answers.
\item[H4 (descriptive).] AUROC ordering designed $<$ agreement alone $\le$ agreement-first
$\le$ v3.
\end{description}

What the bundle hash certifies is the rule code, the thresholds, the truncation policy and the
exclusion rule. It does not certify the answering harness, the prompts or the labeller, which were
rebuilt after the registration (the deviations are listed below). The registration was committed
to what was then a private repository and was not publicly timestamped, so its date rests on that
repository's history. H1 is also a weaker test than its status as primary suggests: the BIRD
decomposition (Table~\ref{tab:stages}) had already located the loss, so H1 mainly asks whether the
same loss recurs on new data. We report the test as a confirmation of the frozen rule's behaviour,
not as a strong pre-registered trial.

\subsection{Results}

The evaluation printed the frozen bundle hash and reproduced the stored designed score and gate
state on all \spiderN{} kept answers. Of \spiderRead{} answers, \spiderDropped{} were left out
as unclear or question-ambiguous; \spiderWrong{} of the kept answers are labelled wrong and
accuracy is \spiderAcc.

\begin{table}[t]
\centering\footnotesize\setlength{\tabcolsep}{4pt}
\caption{Spider dev, pre-registered confirmation, \spiderN{} labelled answers.}\label{tab:spider}
\begin{tabular}{lccccc}
\toprule
Score & AUROC (95\% CI) & \makecell{\pass\\accuracy} & \makecell{\pass\\share} & \makecell{Correct\\withheld} & \makecell{Wrong\\withheld}\\
\midrule
Designed & \spiderDesignedAuroc{} (\spiderDesignedAurocCI) & \spiderDesignedPassAcc & \spiderDesignedPassShare & \spiderDesignedCorrectWithheld & \spiderDesignedWrongWithheld\\
Agreement-first (primary) & \spiderAgreeFirstAuroc{} (\spiderAgreeFirstAurocCI) & \spiderAgreeFirstPassAcc & \spiderAgreeFirstPassShare & \spiderAgreeFirstCorrectWithheld & \spiderAgreeFirstWrongWithheld\\
v3 (empty result) & \spiderVthreeAuroc{} (\spiderVthreeAurocCI) & \spiderVthreePassAcc & \spiderVthreePassShare & \spiderVthreeCorrectWithheld & \spiderVthreeWrongWithheld\\
Agreement alone & \spiderAgreeAuroc{} (\spiderAgreeAurocCI) & \spiderAgreePassAcc & \spiderAgreePassShare & \spiderAgreeCorrectWithheld & \spiderAgreeWrongWithheld\\
\midrule
\multicolumn{6}{l}{\emph{Not pre-registered}}\\
v2 rule & \spiderRuleAuroc{} (\spiderRuleAurocCI) & \spiderRulePassAcc & \spiderRulePassShare & \spiderRuleCorrectWithheld & \spiderRuleWrongWithheld\\
Fitted & \spiderFittedAuroc{} (\spiderFittedAurocCI) & \spiderFittedPassAcc & \spiderFittedPassShare & \spiderFittedCorrectWithheld & \spiderFittedWrongWithheld\\
\bottomrule
\end{tabular}
\end{table}

\begin{table}[t]
\centering\footnotesize\setlength{\tabcolsep}{4pt}
\caption{Pre-registered hypotheses on Spider dev.}\label{tab:hyp}
\begin{tabularx}{\linewidth}{lXl}
\toprule
& Observed & Verdict\\
\midrule
H1 & AUROC difference \spiderAgreeFirstDelta{} (95\% CI \spiderAgreeFirstDeltaCI); with truncated
pairs counted as disagreeing, agreement-first AUROC \spiderAgreeFirstSensAuroc & Pass\\
H2 & correct withheld \spiderAgreeFirstCorrectWithheld{} vs \spiderDesignedCorrectWithheld;
\pass{} accuracy \spiderAgreeFirstPassAcc{} vs \spiderDesignedPassAcc{}
(\spiderPassAccGapPts{} points) & Pass\\
H3 & empty result fired on \spiderEmptyFiredN{} answers (\spiderEmptyFires); right
\spiderEmptyRight{} of the time vs \spiderAcc{} overall, \spiderEmptyGapPts{} points below
& Fail\\
H4 & designed \spiderDesignedAuroc{} $<$ agreement-first \spiderAgreeFirstAuroc{} $<$
agreement alone \spiderAgreeAuroc{} $<$ v3 \spiderVthreeAuroc & Not as predicted\\
\bottomrule
\end{tabularx}
\end{table}

\paragraph{H1 and H2 hold.} The primary effect on Spider (\spiderAgreeFirstDelta) is larger
than its in-sample size on the 29 September BIRD run (\birdExpAgreeFirstDelta) and close to the
BIRD rerun under the same harness (\birdReAgreeFirstDelta). The larger Spider effect comes from
the baseline: the designed score fell from \birdExpDesignedAuroc{} on BIRD to
\spiderDesignedAuroc{} on Spider, while agreement-first stayed close (\birdExpAgreeFirstAuroc{}
and \spiderAgreeFirstAuroc). It is not evidence that the rule generalises better than it fits.
Agreement-first withholds about a fifth as many correct answers as the designed score without
losing \pass{} accuracy; that accuracy, \spiderAgreeFirstPassAcc, should be read against
Spider's base rate of \spiderAcc. H2 does not hold the share of wrong answers withheld fixed, and
that share also fell, from \spiderDesignedWrongWithheld{} to \rtSpiderAfWW, so part of the drop
in correct answers withheld is a lower abstention rate. The designed score cannot be set to
withhold exactly \rtSpiderAfWW{} of wrong answers, because many answers tie at its 0.59 cap: its
nearest thresholds withhold \rtSpiderDzLoWW{} or \rtSpiderDzHiWW{} of wrong answers, and
\rtSpiderDzLoCW{} or \rtSpiderDzHiCW{} of correct ones. Interpolating between them (breaking the
tie at random), the designed score would withhold about \rtSpiderDzCWMatched{} of correct answers
at the same share of wrong answers withheld, against \rtSpiderAfCW{} for agreement-first; on the
BIRD rerun, \rtBirdReDzCWMatched{} against \rtBirdReAfCW. AUROC and Figure~\ref{fig:rc} compare
the scores without a threshold. H2's
pre-registered criterion is a point estimate; the paired bootstrap interval for the change in
\pass{} accuracy is \abSpiderPadDzCI{} points, so a non-inferiority claim at a two-point margin
also holds on the interval.

\paragraph{H3 fails.} On both BIRD runs an empty final result was a strong sign of a wrong
answer (right \birdExpEmptyRight{} and \birdReEmptyRight{} of the time). On Spider it fired more
often and the answers were right \spiderEmptyRight{} of the time. We have not established why.
One candidate is that some Spider questions have empty correct answers over its small databases;
another is that the labeller treats an empty result that matches an empty reference as right.
We report the failure as pre-registered and do not adjust the rule.

\paragraph{H4 reverses one step.} Agreement alone ranked above agreement-first on Spider, the
opposite of BIRD, and v3 ranked above both. In point estimate the pre-registered primary rule is
the third of the three agreement-based rules on the confirmatory data. None of these differences
is resolved: agreement alone minus agreement-first is \rvSpiderAgMinusAf{} (95\% CI
\rvSpiderAgMinusAfCI), and v3 minus agreement alone \rvSpiderVthreeMinusAg{}
(\rvSpiderVthreeMinusAgCI). Agreement alone also withholds fewer correct answers
(\spiderAgreeCorrectWithheld{} against \spiderAgreeFirstCorrectWithheld) and passes more
(\spiderAgreePassShare{} against \spiderAgreeFirstPassShare) at about the same \pass{} accuracy
(\spiderAgreePassAcc{} against \spiderAgreeFirstPassAcc). One explanation, untested, is that the caps agreement-first adds lower some
correct answers that all five samples agreed on, and on easier data with fewer wrong answers that
costs more ranking than it gains.

\paragraph{Deviations from the pre-registration.} The answering prompt is the rebuilt trim
(\texttt{prompt\_sha d86e574ca9c8}), not \texttt{6566415c9593}, because the 29 September
trimmed file was lost with the harness. The judge and labelling prompts are the rebuilt ones.
The rule code, thresholds, truncation policy and exclusion rule are unchanged, which is what the
bundle hash checks. The pre-registration also named commit \texttt{cb66ea7}; the evaluation used the same bundle
from a later commit of the harness repository that adds the rebuilt harness. H2's criterion
was written as a point estimate; the interval reported above was added afterwards. The
registration named 100 wrong answers as the adequacy floor and the test reached
\spiderWrong.

\paragraph{Limits of this test.} Only \spiderWrong{} answers are labelled wrong, fewer than the
100 the pre-registration named as the point below which intervals are too wide to say much, and
the intervals are wide. Spider is easier than BIRD for this model.

\section{Replication and labeller sensitivity}\label{sec:labeller}

\paragraph{BIRD under the rebuilt harness.} The full rerun gave accuracy \birdReAcc{} on
\birdReN{} labelled answers. Agreement-first improved AUROC over the designed score by
\birdReAgreeFirstDelta{} (95\% CI \birdReAgreeFirstDeltaCI), from \birdReDesignedAuroc{} to
\birdReAgreeFirstAuroc, and withheld \birdReAgreeFirstCorrectWithheld{} of correct answers
against \birdReDesignedCorrectWithheld. Execution match was \birdReExecMatch{}, against
\birdExpExecMatch{} on 29 September, so the rebuilt harness answers BIRD about as well as the
original.

\paragraph{Sensitivity to label version.} Relabelling the 29 September BIRD
answers with the rebuilt labeller gave \lblAgree{} agreement with the original labels on
\lblN{} answers both kept ($\kappa=\lblKappa$). The rebuilt labeller is more lenient: it moved
\lblUp{} answers from wrong to right and \lblDown{} the other way, and accuracy rose from
\birdExpAcc{} to \birdRlAcc. Under the rebuilt labels the agreement-first gain on the same
answers is \birdRlAgreeFirstDelta{} (95\% CI \birdRlAgreeFirstDeltaCI), against
\birdExpAgreeFirstDelta{} under the original labels. The direction of the comparison was preserved
under both label versions. Two model-generated label sets agreeing is not independent validation
of either, and unvalidated labels affect every comparison in this paper, not only absolute
accuracy: AUROC, \pass{} accuracy and the judgement of whether a check helps all depend on them.

\paragraph{Under execution-match labels.} The benchmarks' own rule (identical result sets) is
too strict for answers written in an application's style, but it involves no model judgement.
Scored against it on every answer, AUROC was \abSpiderEmDz{} for the designed score,
\abSpiderEmAg{} for agreement alone and \abSpiderEmAf{} for agreement-first on Spider
(\abSpiderEmN{} answers, \abSpiderEmAcc{} matching), \abBirdReEmDz, \abBirdReEmAg{} and
\abBirdReEmAf{} on the BIRD rerun, and \abBirdExpEmDz, \abBirdExpEmAg{} and \abBirdExpEmAf{} on
the 29 September run. Both agreement-based scores rank above the designed score under these labels
too. Their order relative to each other is less stable. On Spider, agreement alone ranks higher
under the model-judged labels (\spiderAgreeAuroc{} against \spiderAgreeFirstAuroc), and
agreement-first ranks higher under execution match (\abSpiderEmAf{} against \abSpiderEmAg). The
two protocols also keep different answer sets, so we repeated the execution-match comparison on
the \spiderN{} answers the model-judged protocol keeps: \abSpiderEmKDz{} for the designed score,
\abSpiderEmKAg{} for agreement alone and \abSpiderEmKAf{} for agreement-first. The reversal
persists on the same answers, so it follows from the label definition. It is a reversal in point
estimates only. On these answers agreement alone minus agreement-first is \rvSpiderEmAgMinusAf{}
(95\% CI \rvSpiderEmAgMinusAfCI) under execution match and \rvSpiderAgMinusAf{}
(\rvSpiderAgMinusAfCI) under model-judged labels; neither interval excludes zero by a clear
margin. Agreement-first minus the designed score is \rvSpiderEmAfMinusDz{}
(\rvSpiderEmAfMinusDzCI) under execution match, and v3 minus agreement alone is
\rvSpiderEmVthreeMinusAg{} (\rvSpiderEmVthreeMinusAgCI). One possible reason,
untested, is that execution match counts an answer that differs only in form as wrong and the
checks fire on some of those answers. The checks' value for predicting correctness therefore depends on what is counted as
correct; their coverage cost does not.

\paragraph{A human pilot.} Two annotators, working separately and blind to every score, to the
model labels and to each other, labelled ten Spider answers: four whose reference the labeller had
judged wrong, three on which agreement alone and agreement-first make different gate decisions,
and three drawn at random. They agreed with each other and with the model label on all ten. Ten
answers are too few to estimate label error or to support an agreement statistic such as Cohen's
$\kappa$ \citep{cohen1960kappa}, and four of the ten come from the group where a lenient labeller
would show. The pilot shows the protocol works and finds no disagreement; it does not validate the
labels. A full audit, with the sample described in the released \texttt{score\_handcheck.py}, is
future work, and the results here should be read as resting on model-judged labels.

\begin{table}[t]
\centering\footnotesize\setlength{\tabcolsep}{4pt}
\caption{AUROC of agreement-first against the designed score in every evaluation.}\label{tab:all}
\begin{tabular}{lccccc}
\toprule
Evaluation & $n$ & Acc. & Designed & Agree-first & Difference (95\% CI)\\
\midrule
BIRD, 29 Sep (in-sample) & \birdExpN & \birdExpAcc & \birdExpDesignedAuroc & \birdExpAgreeFirstAuroc & \birdExpAgreeFirstDelta{} (\birdExpAgreeFirstDeltaCI)\\
BIRD, 29 Sep, rebuilt labels & \birdRlN & \birdRlAcc & \birdRlDesignedAuroc & \birdRlAgreeFirstAuroc & \birdRlAgreeFirstDelta{} (\birdRlAgreeFirstDeltaCI)\\
BIRD, rebuilt harness & \birdReN & \birdReAcc & \birdReDesignedAuroc & \birdReAgreeFirstAuroc & \birdReAgreeFirstDelta{} (\birdReAgreeFirstDeltaCI)\\
Spider (pre-registered) & \spiderN & \spiderAcc & \spiderDesignedAuroc & \spiderAgreeFirstAuroc & \spiderAgreeFirstDelta{} (\spiderAgreeFirstDeltaCI)\\
\bottomrule
\end{tabular}
\end{table}

\begin{table}[t]
\centering\footnotesize\setlength{\tabcolsep}{4pt}
\caption{Headline AUROC differences with 95\% intervals from a database-level cluster bootstrap
(model-judged labels). Answer-level intervals are given in the text.}\label{tab:cluster}
\begin{tabular}{lccc}
\toprule
Difference & BIRD, 29 Sep & BIRD, rebuilt & Spider\\
\midrule
Agree-first $-$ designed & \rtBirdExpClAfDz{} (\rtBirdExpClAfDzCI) & \rtBirdReClAfDz{} (\rtBirdReClAfDzCI) & \rtSpiderClAfDz{} (\rtSpiderClAfDzCI)\\
Agreement $-$ agree-first & \rtBirdExpClAgAf{} (\rtBirdExpClAgAfCI) & \rtBirdReClAgAf{} (\rtBirdReClAgAfCI) & \rtSpiderClAgAf{} (\rtSpiderClAgAfCI)\\
v3 $-$ agreement & \rtBirdExpClVtAg{} (\rtBirdExpClVtAgCI) & \rtBirdReClVtAg{} (\rtBirdReClVtAgCI) & \rtSpiderClVtAg{} (\rtSpiderClVtAgCI)\\
\bottomrule
\end{tabular}
\end{table}

\paragraph{The labeller and the reference.} On Spider the rebuilt labeller judged
\spiderRefWrong{} reference answers wrong and marked \spiderRefWrongMarkedCorrect{} of those
answers correct. If the labeller is too lenient, this is where it would show, and the hand check
draws part of its sample from this group.

\section{What the gain is made of}\label{sec:ablate}

These analyses were not pre-registered. They use the stored samples and labels, with no model
calls (\texttt{analyze\_ablation.py}).

\paragraph{Where the designed score loses its ranking.} Adding a constant to a score or scaling
it changes calibration and decisions at fixed thresholds, but not the ordering of answers, so it
cannot change AUROC. Ranking is lost only where a step reorders answers or creates ties.
Table~\ref{tab:stages} rebuilds the designed score one step at a time. The designed agreement
term $L_2$ alone ranks well. The large loss comes at the next step, $\mathrm{SAFETY}=\min(L_1,L_2)$:
$L_1$ drops to zero whenever the judge reports an unsupported claim, including claims in
commentary, and the minimum then discards $L_2$ for those answers. The weighted mean recovers a
little, because the other terms vary on some answers, and the hard caps cost a little more by
tying \abSpiderTiesAtCap{} Spider answers at exactly 0.59.

\begin{table}[t]
\centering\footnotesize\setlength{\tabcolsep}{4pt}
\caption{AUROC of the designed score rebuilt one step at a time.}\label{tab:stages}
\begin{tabular}{lcccc}
\toprule
Evaluation & $L_2$ alone & $\min(L_1,L_2)$ & Weighted mean & With caps (designed)\\
\midrule
BIRD, 29 Sep & \abBirdExpStL & \abBirdExpStSafety & \abBirdExpStWeighted & \abBirdExpStDesigned\\
BIRD, rebuilt harness & \abBirdReStL & \abBirdReStSafety & \abBirdReStWeighted & \abBirdReStDesigned\\
Spider & \abSpiderStL & \abSpiderStSafety & \abSpiderStWeighted & \abSpiderStDesigned\\
\bottomrule
\end{tabular}
\end{table}

\paragraph{The statistic under one matching policy.} The designed $L_2$ and the agreement share
$\Agr$ differ in two ways: the statistic (entropy over clusters against the share agreeing with
the shown answer), and the matching policy (full-row signatures against projection matching
with four-place rounding). Holding the matching relation fixed and changing only the statistic
isolates Proposition~\ref{prop:entropy}'s practical effect. Strict matching (same columns, same
multiset of rows) is an equivalence on results, so its clusters are well defined. Relaxed matching
is not transitive: a one-column result can match a two-column result through its first column,
and that result can match another one-column result through its second, while the two one-column
results do not match each other. For relaxed entropy we therefore cluster greedily in sample
order, each sample joining the first cluster whose first member it matches; clustering by
connected components of the match relation instead changed the clusters of
\rvBirdExpEntDiffN, \rvBirdReEntDiffN{} and \rvSpiderEntDiffN{} answers and moved AUROC from
\rvBirdExpEntGreedy{} to \rvBirdExpEntComp, \rvBirdReEntGreedy{} to \rvBirdReEntComp{} and
\rvSpiderEntGreedy{} to \rvSpiderEntComp. The relaxed comparison is still reported as
secondary. With strict matching, $\Agr$ minus the entropy score was \abBirdExpDStrict{} AUROC
(95\% CI \abBirdExpDStrictCI) on the 29 September BIRD run, \abBirdReDStrict{}
(\abBirdReDStrictCI) on the BIRD rerun and \abSpiderDStrict{} (\abSpiderDStrictCI) on Spider; with
relaxed matching, \abBirdExpDRelaxed, \abBirdReDRelaxed{} and \abSpiderDRelaxed. The proposition
is true, but the case it describes, a shown answer outside a large cluster, is rare enough that
the choice of statistic moves AUROC by less than 0.02. The gain over the designed score comes
from not taking the minimum with a checklist that fires on commentary.

\begin{table}[t]
\centering\footnotesize\setlength{\tabcolsep}{4pt}
\caption{Gate decisions by correctness (right / wrong). \flag{} answers are shown with a warning;
\abstain{} answers are withheld.}\label{tab:decisions}
\begin{tabular}{llccc}
\toprule
Evaluation & Score & \pass & \flag & \abstain\\
\midrule
Spider & Designed & \abSpiderDzPassRight{} / \abSpiderDzPassWrong & \abSpiderDzFlagRight{} / \abSpiderDzFlagWrong & \abSpiderDzAbstRight{} / \abSpiderDzAbstWrong\\
 & Agreement alone & \abSpiderAgPassRight{} / \abSpiderAgPassWrong & \abSpiderAgFlagRight{} / \abSpiderAgFlagWrong & \abSpiderAgAbstRight{} / \abSpiderAgAbstWrong\\
 & Agreement-first & \abSpiderAfPassRight{} / \abSpiderAfPassWrong & \abSpiderAfFlagRight{} / \abSpiderAfFlagWrong & \abSpiderAfAbstRight{} / \abSpiderAfAbstWrong\\
\midrule
BIRD, rebuilt & Designed & \abBirdReDzPassRight{} / \abBirdReDzPassWrong & \abBirdReDzFlagRight{} / \abBirdReDzFlagWrong & \abBirdReDzAbstRight{} / \abBirdReDzAbstWrong\\
 & Agreement alone & \abBirdReAgPassRight{} / \abBirdReAgPassWrong & \abBirdReAgFlagRight{} / \abBirdReAgFlagWrong & \abBirdReAgAbstRight{} / \abBirdReAgAbstWrong\\
 & Agreement-first & \abBirdReAfPassRight{} / \abBirdReAfPassWrong & \abBirdReAfFlagRight{} / \abBirdReAfFlagWrong & \abBirdReAfAbstRight{} / \abBirdReAfAbstWrong\\
\bottomrule
\end{tabular}
\end{table}

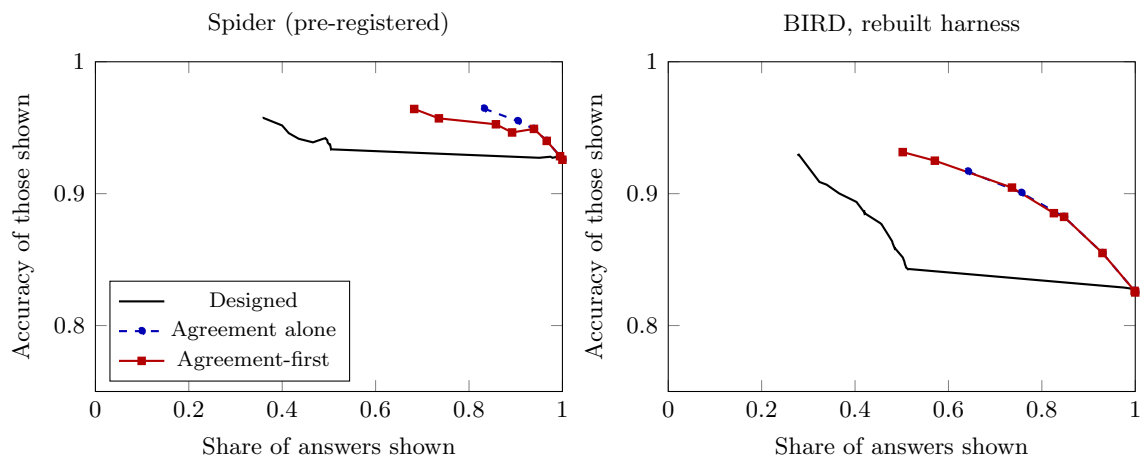
\begin{figure}[t]
\centering
\begin{tikzpicture}
\begin{groupplot}[group style={group size=2 by 1, horizontal sep=1.4cm},
  width=0.47\linewidth, height=0.36\linewidth, xmin=0, xmax=1, ymin=0.75, ymax=1,
  xlabel={Share of answers shown}, ylabel={Accuracy of those shown},
  tick label style={font=\footnotesize}, label style={font=\footnotesize},
  title style={font=\footnotesize}, legend style={font=\scriptsize, at={(0.03,0.03)}, anchor=south west}]
\nextgroupplot[title={Spider (pre-registered)}]
\addplot[black, thick, mark=none] table {rc_spider_designed.dat};
\addplot[blue!70!black, thick, dashed, mark=*, mark size=1.2pt] table {rc_spider_agreement.dat};
\addplot[red!70!black, thick, mark=square*, mark size=1.2pt] table {rc_spider_agree_first.dat};
\legend{Designed, Agreement alone, Agreement-first}
\nextgroupplot[title={BIRD, rebuilt harness}]
\addplot[black, thick, mark=none] table {rc_birdRe_designed.dat};
\addplot[blue!70!black, thick, dashed, mark=*, mark size=1.2pt] table {rc_birdRe_agreement.dat};
\addplot[red!70!black, thick, mark=square*, mark size=1.2pt] table {rc_birdRe_agree_first.dat};
\end{groupplot}
\end{tikzpicture}
\caption{Accuracy and coverage: accuracy of the answers shown as the threshold is lowered. Markers,
and every vertex of the designed score's curve, are attainable thresholds; the lines between them
are visual guides.}\label{fig:rc}
\end{figure}

\paragraph{What the checks do on top of agreement.} Agreement-first is agreement alone with
vetoes and caps added. Because the checks can only lower a score, they move answers out of
\pass{} and never into it. Table~\ref{tab:decisions} gives the counts. On Spider the checks
removed \abSpiderHeld{} answers from \pass: \abSpiderHeldRightN{} correct and
\abSpiderHeldWrongN{} wrong. All \abSpiderMovedFlagWrong{} wrong ones went to \flag, where they are
still shown with a warning; the number of wrong answers withheld was unchanged, while
\abSpiderMovedAbstRight{} more correct answers were withheld. \pass{} accuracy changed by
\abSpiderPadAg{} points (95\% CI \abSpiderPadAgCI). On BIRD the checks do more: on the rerun they
removed \abBirdReHeld{} answers from \pass{} (\abBirdReHeldWrongN{} wrong), raising \pass{}
accuracy by \abBirdRePadAg{} points (\abBirdRePadAgCI), with \abBirdReMovedAbstWrong{} wrong and
\abBirdReMovedAbstRight{} correct answers moved directly from \pass{} under agreement alone to
\abstain; the net rise in correct answers withheld also includes answers that agreement alone had
already flagged. On these benchmarks the
checks mainly turn passes into warnings and cost unqualified coverage; they rarely withhold a
wrong answer that agreement would have shown. Whether they do more on records where tables are
missing, joins are ambiguous and codes stand in for values is the question their design was for,
and these benchmarks do not answer it.

\paragraph{Accuracy and coverage.} The gate's thresholds were set by design and not fitted.
Figure~\ref{fig:rc} shows, for every threshold, the share of answers shown and their accuracy.
Agreement alone and agreement-first lie above the designed score over most of the range; at the
lowest coverages on BIRD the three are close. The discrete agreement share gives only a few
operating points.

\paragraph{The judge as a signal.} A GPT-5 mini judge already reads every answer: it lists the
question's parts, the parts the answer covers, and any issues, each classed by kind and by whether
it concerns the answer or the commentary. Before computing any AUROC we fixed two scores from its
stored output: the share of parts covered minus the number of answer-scope issues, and the same
with issues of any scope. No new model call was made. Table~\ref{tab:judge} gives the result. The
judge ranks answers well below agreement in all three evaluations, and adding it to agreement as
a tie-breaker within agreement levels gains little. This judge was built to critique, not to say
whether the answer is right, and it never sees the reference.

\paragraph{A correctness judge.} The first review asked for the baseline the critique judge is
not: a model asked directly how likely an answer is to be right, as in
\citet{richardson2026predicts}. After that review, and without pre-registration, we ran one on
every stored answer. GPT-5.1 saw the question, the hint, the queries with up to 20 rows each and
the written answer, never the reference, the other samples or any score, and returned a
probability (\texttt{bench/judge\_correct.py}). Table~\ref{tab:pcorrect} gives the result. On
its own the judge ranked above agreement on Spider, level with it on the 29 September BIRD run
and below it on the BIRD rerun. The plain average of the two, with no weight fitted, ranked above
agreement in all three evaluations, and the intervals stay above zero under the database-level
bootstrap (\rjBirdExpClMeanMinusAgCI, \rjBirdReClMeanMinusAgCI{} and \rjSpiderClMeanMinusAgCI).

Three things limit this result. The judge is the model that wrote the answers, and the labels
come from another model from the same provider, so part of the agreement between judge and labels
may be shared habit; under execution-match labels the judge ranks below agreement in every
evaluation (on Spider \rjSpiderEmJudge{} against \rjSpiderEmAg), and the average's gain over
agreement shrinks to \rjBirdExpEmMeanMinusAg{} (\rjBirdExpEmMeanMinusAgCI),
\rjBirdReEmMeanMinusAg{} (\rjBirdReEmMeanMinusAgCI) and \rjSpiderEmMeanMinusAg{}
(\rjSpiderEmMeanMinusAgCI). Its probabilities are not calibrated: its Brier score is worse than a
constant at the base rate in every evaluation (\rjSpiderBrier{} against \rjSpiderBrierConst{} on
Spider). And the average was chosen after seeing the critique-judge result, not fixed in
advance. Whether a judge from another model family gives the same gain is untested.

\begin{table}[t]
\centering\footnotesize\setlength{\tabcolsep}{3pt}
\caption{A correctness judge (GPT-5.1 asked for the probability that the answer is right) against
agreement, model-judged labels, answer-level 95\% intervals. Added after review; not
pre-registered.}\label{tab:pcorrect}
\begin{tabular}{lccccc}
\toprule
Evaluation & \makecell{Judge\\(95\% CI)} & Agreement & \makecell{Judge $-$\\agreement} & Average & \makecell{Average $-$\\agreement}\\
\midrule
BIRD, 29 Sep & \rjBirdExpAuroc{} (\rjBirdExpAurocCI) & \birdExpAgreeAuroc & \makecell{\rjBirdExpMinusAg\\(\rjBirdExpMinusAgCI)} & \rjBirdExpMean & \makecell{\rjBirdExpMeanMinusAg\\(\rjBirdExpMeanMinusAgCI)}\\
BIRD, rebuilt & \rjBirdReAuroc{} (\rjBirdReAurocCI) & \birdReAgreeAuroc & \makecell{\rjBirdReMinusAg\\(\rjBirdReMinusAgCI)} & \rjBirdReMean & \makecell{\rjBirdReMeanMinusAg\\(\rjBirdReMeanMinusAgCI)}\\
Spider & \rjSpiderAuroc{} (\rjSpiderAurocCI) & \spiderAgreeAuroc & \makecell{\rjSpiderMinusAg\\(\rjSpiderMinusAgCI)} & \rjSpiderMean & \makecell{\rjSpiderMeanMinusAg\\(\rjSpiderMeanMinusAgCI)}\\
\bottomrule
\end{tabular}
\end{table}

\paragraph{A second model family.} To test whether these signals depend on one provider, we
asked Claude Sonnet 4.5 each Spider question once, with the same tools and the same trimmed
system prompt, and compared its final result with the scored GPT-5.1 answer using the relaxed
match of Section~\ref{sec:score} (\texttt{crossmodel.py}; \rxN{} answers measured, the rest
undetermined by truncation). This was run after review and is not pre-registered. The second
model's answer differed from the scored answer on \rxFires{} of answers, and those were right
\rxRightWhenFired{} of the time against \spiderAcc{} overall. Used as a tie-break inside
agreement levels, it raised AUROC from \spiderAgreeAuroc{} to \rxAgreeCross{}
(\rxAgreeCrossDelta, database-level 95\% CI \rxAgreeCrossDeltaClCI). Unlike the correctness
judge, the gain held under execution-match labels (\rxEmAgreeCrossDelta, \rxEmAgreeCrossDeltaCI).
Added to the average of agreement and the correctness judge it changed little
(\rxMeanJudgeCross, difference \rxTrioMinusMeanCI). The frozen v3 code from 29 September already
carried this signal as a cap at \flag; evaluated by the frozen bundle with the empty-result and
cross-model signals, v3 reached AUROC \rxVthreeAuroc{} (\rxVthreeAurocCI), \rxVthreeDelta{}
over the designed score (\rxVthreeDeltaCI), with \pass{} accuracy \rxVthreePassAcc{} on
\rxVthreePassShare{} of answers, withholding \rxVthreeWrongWithheld{} of wrong answers and
\rxVthreeCorrectWithheld{} of correct ones (agreement-first on the same answers:
\rxAfSameAuroc, \rxAfSameWrongWithheld{} and \rxAfSameCorrectWithheld). The cost is a second
answering run per question; the second model's own answers are not labelled here.

\begin{table}[t]
\centering\footnotesize\setlength{\tabcolsep}{4pt}
\caption{The pipeline's own judge scored as a standalone signal, model-judged labels. The judge
score is the share of question parts covered minus the number of answer-scope issues.}\label{tab:judge}
\begin{tabular}{lcccc}
\toprule
Evaluation & \makecell{Judge\\(95\% CI)} & Agreement & \makecell{Judge $-$ agreement\\(95\% CI)} & \makecell{Agreement, judge\\as tie-break}\\
\midrule
BIRD, 29 Sep & \rvBirdExpJudgeAuroc{} (\rvBirdExpJudgeAurocCI) & \birdExpAgreeAuroc & \rvBirdExpJudgeMinusAgree{} (\rvBirdExpJudgeMinusAgreeCI) & \rvBirdExpAgreeJudgeAuroc\\
BIRD, rebuilt harness & \rvBirdReJudgeAuroc{} (\rvBirdReJudgeAurocCI) & \birdReAgreeAuroc & \rvBirdReJudgeMinusAgree{} (\rvBirdReJudgeMinusAgreeCI) & \rvBirdReAgreeJudgeAuroc\\
Spider & \rvSpiderJudgeAuroc{} (\rvSpiderJudgeAurocCI) & \spiderAgreeAuroc & \rvSpiderJudgeMinusAgree{} (\rvSpiderJudgeMinusAgreeCI) & \rvSpiderAgreeJudgeAuroc\\
\bottomrule
\end{tabular}
\end{table}

\paragraph{Calibration.} None of these scores was designed as a probability, and the agreement
share takes only six values. Table~\ref{tab:calib} gives expected calibration error and Brier
score for completeness. Agreement-based scores are better calibrated than the designed score in
every evaluation. That is a relative claim. On Spider, where \spiderAcc{} of answers are right, a
constant prediction at the base rate has a Brier score of \rvSpiderBrierConst; agreement-first
(\rvSpiderBrierAgreeFirst) does worse than that and agreement alone (\rvSpiderBrierAgree) only a
little better. On the BIRD rerun the constant gives \rvBirdReBrierConst{} against
\rvBirdReBrierAgreeFirst{} for agreement-first. A deployment that wants a probability should fit one
on its own labelled answers.

\paragraph{Cost of sampling.} Agreement needs $K=5$ answers per question. On the complete BIRD
rerun the answering model made \rvTokCallsPerQ{} calls per question, with about \rvTokInPerQ{}
input tokens (about \rvTokInPerSample{} per sample, mostly the system prompt and tool results
re-read at each step) and \rvTokOutPerQ{} output tokens; the judge added \rvTokJudgeInPerQ{} input
and \rvTokJudgeOutPerQ{} output tokens. Wall-clock time per question, with five samples run in
parallel on a discounted, slower service tier and eight questions at a time, had a median of
\rvBirdReSecMedian{} seconds (90th percentile \rvBirdReSecPninety) on BIRD and
\rvSpiderSecMedian{} seconds (\rvSpiderSecPninety) on Spider. Five samples cost about five times
one answer in tokens; latency depends on whether the samples run in parallel.

\begin{table}[t]
\centering\footnotesize\setlength{\tabcolsep}{4pt}
\caption{Expected calibration error (10 bins) and Brier score.}\label{tab:calib}
\begin{tabular}{lcccccc}
\toprule
& \multicolumn{2}{c}{Designed} & \multicolumn{2}{c}{Agreement alone} & \multicolumn{2}{c}{Agreement-first}\\
Evaluation & ECE & Brier & ECE & Brier & ECE & Brier\\
\midrule
BIRD, 29 Sep & \birdExpDesignedEce & \birdExpDesignedBrier & \birdExpAgreeEce & \birdExpAgreeBrier & \birdExpAgreeFirstEce & \birdExpAgreeFirstBrier\\
BIRD, rebuilt harness & \birdReDesignedEce & \birdReDesignedBrier & \birdReAgreeEce & \birdReAgreeBrier & \birdReAgreeFirstEce & \birdReAgreeFirstBrier\\
Spider & \spiderDesignedEce & \spiderDesignedBrier & \spiderAgreeEce & \spiderAgreeBrier & \spiderAgreeFirstEce & \spiderAgreeFirstBrier\\
\bottomrule
\end{tabular}
\end{table}

\section{Discussion}

\paragraph{What the score can claim.} On two public benchmarks with one model, letting agreement
among samples decide the gate ranks answers better than the designed weighted score and withholds
far fewer correct answers at the same or better \pass{} accuracy. These results do not single out
the frozen rule among agreement-based rules: on Spider agreement alone and v3 ranked at least as
well, within the intervals. The claim that a designed score
of 0.80 means ``right about 80\% of the time'' does not hold across the scale and we do not make
it.

\paragraph{Agreement, not the statistic.} Proposition~\ref{prop:entropy} identifies a real blind
spot in the designed entropy term, but under a common matching policy it costs less than 0.02
AUROC (Section~\ref{sec:ablate}). The designed score lost its ranking where agreement met the
checklist through a minimum, and the checklist fired on commentary. Agreement with the shown
answer remains the better quantity to gate on, since it is the one the gate is about, but the
evidence for preferring it over the designed entropy is modest.

\paragraph{The checks have not yet earned their coverage.} The deterministic checks catch every
planted error of their kind. On real answers from these benchmarks they and the judge caps rarely
fire on wrong answers that agreement missed. On Spider they moved \abSpiderHeldRightN{} correct
and \abSpiderHeldWrongN{} wrong answers out of \pass, put the wrong ones behind a warning, and
withheld no additional wrong answer. They were designed for extracted,
multi-source records where tables, joins and codes go wrong, and their case has to be made on
such records. Poor discrimination of core-answer correctness also does not make a check useless
for other governance purposes, such as flagging unsupported commentary. It suggests that those
purposes should not be folded into the same confidence number; part of the design problem may be
asking one scalar to carry several kinds of assurance.

\paragraph{Commentary is a product fault.} The prompt's instruction to enrich every answer
produces unsupported commentary in \exFireUnsupportedCommentaryShare{} of BIRD answers. Capping
unsupported claims only when they concern the answer raises the designed score's \pass{} share
from \exEDesignedPassShare{} to \exEUnsupportedAnswerOnlyPassShare{} at the same \pass{} accuracy
(\exEUnsupportedAnswerOnlyPassAcc); ignoring commentary entirely leaves the user reading it. The
frozen rules record it and name it without changing the gate.

\paragraph{Confident misreadings remain.} When all five samples agree and are wrong, agreement
cannot help. The empty-result signal was meant for part of this and did not meet its
pre-registered effect-size criterion on Spider.
This is the open problem the rules leave, and Section~\ref{sec:next} argues that it is a
problem of context.

\section{Next steps: context}\label{sec:next}

This section is forward-looking. Nothing in it is measured in this paper.

\paragraph{The limitation.} Agreement among samples shows whether the model reads a question the
same way each time. It cannot resolve an intended business definition that is absent from the
evidence the model and the checker can see. If every sample makes the same choice of meaning,
agreement is complete whether or not that meaning is the one the organisation uses, and the
deterministic checks can all pass: the tables exist, the joins pass the implemented overlap check
and every number is a returned cell.
\ifstagingcase Appendix~\ref{app:case} describes such a case from a staging
implementation over a government organisation's hazardous-materials records. Asked how much
water was used, the system read storage-reporting tables and reported the recorded weight of
products in which water is listed; with an expert-approved definition of usage it followed the
issue transactions through each product's composition and reported the weight of water in what
was issued. The two recorded outputs cover different reporting periods, so their numerical
difference illustrates the error and does not estimate its size.
\fi

\paragraph{Definitions as checks.} The case suggests a use for approved definitions beyond
improving the answer. A recorded rule such as ``usage questions follow the issue transactions;
the storage-reporting tables answer only storage questions'' can be tested against a trace
independently of how the answer was produced: an answer about usage whose queries read the
storage-reporting tables breaks it, and the gate can name the rule. Such checks would sit beside
the existing vetoes. Whether they catch confident misreadings without withholding many correct
answers is an empirical question, and Section~\ref{sec:ablate} shows that checks can cost more
coverage than they repay.

\paragraph{The experiment.} Context can help in two places, the answering system and the
checking system, and a single before-and-after case cannot tell them apart. The test we will run
crosses the two on held-out questions in four conditions: no context for either; context for
answer generation only; context for checking only; and context for both. In the checking-only
condition the generated answer and its trace are held fixed, so any gain is the checker's alone.
The context rules are frozen before the held-out questions are evaluated. We will report answer
correctness, wrong answers passed, correct answers withheld, and the number of wrong answers that
arrived with full agreement. BIRD's evidence hints give a public version of the same contrast,
since each hint is a recorded definition. Every answer in this paper was generated with the hint,
so the contrast needs a second answering run without hints; it is the first experiment we will
run.

\paragraph{In deployment.} Until that test is done, the measurements here support a narrower
practice. A \pass{} should require that at least four of five samples reached the same result.
Whether the checks may also block a \pass{} is a coverage decision, not a ranking one: on Spider
they moved \abSpiderHeldRightN{} correct and \abSpiderHeldWrongN{} wrong answers from \pass{} to a
warning, and agreement alone ranked as well as the rule with checks. Every \flag{} or \abstain{}
should name what caused it so a reviewer can act. Averaging agreement with a judge asked directly
whether the answer is right ranked better than agreement alone in every evaluation, but here the
judge graded its own model's answers, its probabilities are not calibrated and the gain shrinks
under execution-match labels; it is a candidate for the gate that needs a cross-family test, not
yet a recommendation. Agreement between two model families held its gain under both label
definitions on Spider and is the stronger candidate, at the price of a second answering run.
The score should not be presented as a probability. Each answer an expert approves or corrects
is a labelled example in the customer's own terms, and a set of them is what a per-customer
calibration of the thresholds would need.

\section{Limitations}

\begin{itemize}
\item \textbf{Labels are model-judged}, and people have checked only a ten-answer pilot; the two labellers
differ by about ten points of accuracy on the same answers.
\item \textbf{One answering model} (GPT-5.1) and two benchmarks, neither in the system's own
domain. Whether the loss at $\min(L_1,L_2)$ belongs to the score or to this model's habit of
adding commentary needs a second model.
\item \textbf{The correctness judge grades its own model.} It was added after review, is not
pre-registered, is the answering model (GPT-5.1), and shares a provider with the labeller; a judge
from another family was not run. The second-family check (one Claude Sonnet 4.5 answer per
question) was run on Spider only.
\item \textbf{Weak confirmatory design.} The registration fixed the rule and analysis but not the
harness or prompts, H1 was close to a consequence of the BIRD decomposition, and the
pre-registered primary rule was not the best agreement-based rule on Spider.
\item \textbf{The Spider test has \spiderWrong{} wrong answers}, below the 100 the
pre-registration set as adequate.
\item \textbf{The harness was rebuilt} between the exploratory run and the confirmation, and the
answering, judge and labelling prompts differ from the 29 September ones as described; the
scorer and rules do not.
\item \textbf{Reference implementation.} We test the reference scorer written from the design
text, not the product scorer, which differs in places (for example its RETRIEVE term without a
retrieval step).
\item \textbf{Untested parts.} The competence-boundary layer and RETRIEVE are not tested; the
stand-in check is tested only by planting.
\item \textbf{The benchmark rule} asking the last query to match the requested output is not in
the production prompt.
\item \textbf{One run per benchmark.} Run-to-run variance is not measured, beyond the BIRD rerun.
\item \textbf{The checks' value is untested where it should matter.} On these benchmarks the checks
always cost unqualified coverage, and whether they improve ranking depends on the dataset and the
correctness definition.
\item \textbf{Core-answer correctness only.} The labels judge the answer to the question asked;
unsupported commentary is recorded separately and is not what the scores are evaluated against.
\item \textbf{Exploratory analyses.} Sections~\ref{sec:explore} and~\ref{sec:ablate} and the
fitted variant are not pre-registered; only H1 to H4 on Spider are confirmatory.
\end{itemize}

\section{Conclusion}

A governed answer score can be audited on public data at modest cost, and this one did not hold
up as designed. Its deterministic checks catch the errors they are built for; its ranking came
from agreement among samples, and was lost where agreement was combined by a minimum with a
checklist that fired on commentary. A rule frozen before the Spider test that lets agreement decide the gate improved ranking and
coverage of core-answer correctness on Spider as it had on BIRD, though agreement alone did as
well, and both agreement-based scores outperformed
the designed score under every labelling protocol tested. The improvement owes little to the choice
of agreement statistic. Averaging agreement with a correctness judge ranked better still in every
evaluation under model-judged labels; with the judge, the answering model and the labeller from
one provider, that needs a cross-family test before it changes the gate. One answer from a second model family
added ranking under both label definitions on Spider. The added checks reduced unqualified coverage, and their ranking benefit
depended on the dataset and the correctness definition; on Spider under model-judged labels they
moved \abSpiderHeldWrongN{} additional wrong answers to warnings and withheld no additional wrong answer. The
empty-result signal did not meet its pre-registered criterion on Spider. Agreement cannot supply a definition the evidence does not contain; testing whether
approved definitions improve the answers, the checks, or both is the next study.

\section*{Competing interests}

The authors founded and work for the company that develops the system under test. The scorer, the
rules and the pre-registration were fixed before any confirmatory answer existed, and every result
is reported, including those unfavourable to the design.

\section*{Reproducibility}

Code, scripts and data are in the public repository\\\url{https://github.com/numberz-ai/numberz-ies-tdd-h-calibration}.
The branch \texttt{prereg} carries the history of the private repository in which the study was
frozen, pushed unaltered: the scorer bundle \texttt{6e3e723e8f76711d} at commit \texttt{cb66ea7}
(29 September 2026, 10:17 EDT) and the pre-registration, which names that commit and bundle, at
commit \texttt{f462092} (11:46 the same day), before any Spider answer was generated. That
repository was private and its history was made public only on 29 October 2026, so these dates
rest on the commit history and are not independently timestamped. The rebuilt answering, judging
and labelling harness is in \texttt{bench/}. The scripts that compute every number in this
version are in \texttt{paper\_analysis/}; \texttt{evaluate.py} prints the bundle hash and the
hashes of its inputs, and \texttt{make\_numbers.py} and \texttt{make\_review\_numbers.py} refuse
or record any result whose bundle or inputs do not match. The answer, judge and label files are
attached to the \texttt{v2.0} release. In version 1 this section described the repository as it
was intended to be released rather than as it was; this version describes what is public. The
29 September BIRD evaluation was reproduced on a second host (Linux, Python 3.13) with identical
results for every reported score.

\bibliographystyle{plainnat}
\bibliography{refs}

\appendix

\section{Prompt sections removed for the benchmark}\label{app:prompt}

Removed: memory and metadata precedence; document-versus-data routing and vector-first
retrieval, tool discovery and proxying; the vector-database strategy; a domain table guide;
citation and document handling; visualisation rules and chart-tool schemas; tool-output
truncation; protocol-specific SQL repair; workflow creation; forecasting. Kept: identity and
domain, priority hierarchy, core rules, SQL construction rules, the no-vector-database strategy,
interpretation and clarification policies, response format, numeric fidelity, data verification,
name resolution, insight generation, and the self-check. The cut is a list of line ranges with a
reason for each dropped line (\texttt{prompt/trim\_spec.json}); the script only deletes and stops
if the source prompt changes.

\section{Planted errors}\label{app:inject}

\begin{tabularx}{\linewidth}{lX}
\toprule
Error & How it is planted\\
\midrule
Fake table & A table in the final SQL renamed to one that does not exist\\
Bad join & The join rewired to a column whose keys overlap the other side's by less than 0.85 in both directions\\
Stand-in as value & A $-1$ code written into a numeric result column, recorded as that column's stand-in (simulated; BIRD has few)\\
Wrong number & One number in the written answer changed to a value in no returned row\\
Wrong time window & A year or date in the filter shifted by one year\\
Wrong aggregate & SUM$\leftrightarrow$AVG, or COUNT$\leftrightarrow$COUNT DISTINCT\\
Intent flip & Sort direction, MAX$\leftrightarrow$MIN, or a comparison reversed\\
Partial answer & Half the rows, or the last requested column, left out\\
\bottomrule
\end{tabularx}

\section{Fixes to the reference scorer before the BIRD run}

Names defined with \texttt{WITH} are not tables. Join readiness takes the better of the two
directions of key overlap; read literally, the design's rule depends on the order of the
\texttt{ON} clause. The SQLite catalogue and table-valued functions are not fabricated tables;
before the fix they caused false caps in the full run.

\ifstagingcase
\section{A worked case from a staging implementation}\label{app:case}

This appendix describes questions asked of the system in a staging implementation over a
government organisation's hazardous-materials records. The organisation is not named. Table and
column names are generalised and the queries are simplified; they follow the structure of the
queries recorded in the session logs. The figures come from those sessions, are not part of this
study's measurements, cover different reporting periods, and have not been independently
reconciled; they illustrate the kind of error and say nothing about its size.

\subsection*{The records}

Two families of tables describe chemicals here, and they answer different questions. Annual
storage-reporting tables, one per reporting year, were built for the federal chemical-inventory
rule. They record, per site and chemical registry (CAS) number, quantities as signed transaction
deltas in pounds. Issue tables record what was handed out at issue points: one row per issue
transaction, linked to an inventory item, which is linked to a product and its safety data sheet.
The safety data sheet lists the product's constituents, each with a CAS number and a percentage
by weight, and the system reads that section as rows. A chemical that is a constituent of many
products, such as water, appears in the issue history only through those composition rows.

\subsection*{First answer: the wrong quantity}

The user asked how many pounds of CAS 7732-18-5 (water) were used per year. The model read
\emph{used} as a question for the storage-reporting tables, which carry a CAS column and a
pounds column. Its query had this shape, grouped by calendar year of the transaction date:
{\small
\begin{verbatim}
WITH s AS (
  SELECT cas, site_id, tran_date, cas_qty_lbs FROM storage_report_2022
  UNION ALL SELECT cas, site_id, tran_date, cas_qty_lbs FROM storage_report_2023
  UNION ALL SELECT cas, site_id, tran_date, cas_qty_lbs FROM storage_report_2024
  UNION ALL SELECT cas, site_id, tran_date, cas_qty_lbs FROM storage_report_2025
)
SELECT YEAR(tran_date)        AS yr,
       -SUM(cas_qty_lbs)      AS pounds_issued   -- removals are negative deltas
FROM   s
WHERE  cas = '7732-18-5' AND cas_qty_lbs < 0
GROUP  BY YEAR(tran_date);
\end{verbatim}
}
It returned 230,158 pounds for 2022, 192,337 for 2023, 329,949 for 2024 and 4,589 for 2025, and
the answer was written up with a chart. The tables exist, the query has no join, no stand-in
values were returned, and every number in the text is a returned cell; repeated runs made the
same choice. The quantity was the wrong one: these tables report, for storage reporting, the
weight of products in which water is listed, not the water in what was issued.

\subsection*{Second attempt}

Told to use the issue data, the model joined the issue history to a chemical list on an assumed
key, found no rows, and said so: it listed the tables it had checked, did not substitute the
storage-reporting figures, and said that no rows for so common a chemical more likely meant a
broken join path than no use.

\subsection*{The approved definition and the answer under it}

The organisation's subject-matter expert stated what \emph{usage} means there: the weight of the
chemical itself in what was issued, from the issue transactions through each product's safety
data sheet. Model weights and application code were unchanged; the system received this
definition and the routing instruction that follows from it as recorded context. The query
followed the path the schema does not show, from an issue to its product's composition rows. The
analysis was annual; fiscal 2024 is shown:
{\small
\begin{verbatim}
SELECT SUM(h.qty_issued_kg * 2.20462 * sc.weight_pct / 100.0) AS water_lbs
FROM   inventory_issue_history h
JOIN   inventory_details d ON d.inventory_id = h.inventory_id
JOIN   sds_composition  sc ON sc.sds_id      = d.sds_id
WHERE  sc.cas = '7732-18-5'
  AND  h.issue_date >= '2023-10-01' AND h.issue_date < '2024-10-01';   -- fiscal 2024
\end{verbatim}
}
It returned 2,618.94 pounds of water for fiscal 2024. The first answer's 2024 figure covers the
calendar year and a different quantity, so the two are not a same-period comparison. The
displayed query is evidence of the intended calculation, not proof that its total is right. An
independent reconciliation would record the reporting period, the number of transactions
included, the unit normalisation, the version of each safety data sheet used, how missing or
ranged percentages were treated, and the expected join cardinality (one composition row per
product and CAS, so that no issue is counted twice), and would compare the total with one the
expert computes separately. That reconciliation has not been done.

Related sessions on hexavalent chromium surfaced the same kind of problem in other forms:
one substance under several names and registry numbers, identifiers stored with inconsistent
punctuation, a table named for one year holding rows from others, signed quantities, and site
codes without names. Each was fixed by a short recorded rule, not by a change to the model.

\subsection*{What the case shows and what it does not}

Each fix was a short statement of meaning, approved by someone who owns that meaning: which
tables answer which kind of question, what measure a word resolves to, which identifiers name
the same thing, how a quantity's sign is read, and what to do when the evidence is absent. The
case shows that such statements changed the answers. It does not show that they improve the
confidence score's ability to separate right answers from wrong ones; that is the experiment of
Section~\ref{sec:next}.
\fi

\end{document}

%% file: numbers.tex
\newcommand{\birdExpN}{1{,}515}

\newcommand{\birdExpAcc}{72.9\%}
\newcommand{\birdExpWrong}{410}
\newcommand{\birdExpRefWrong}{161}
\newcommand{\birdExpRefWrongPct}{10.5\%}

\newcommand{\birdExpExecMatch}{54.2\%}

\newcommand{\birdExpDesignedAuroc}{0.624}
\newcommand{\birdExpDesignedAurocCI}{0.595--0.652}
\newcommand{\birdExpDesignedPassAcc}{77.9\%}
\newcommand{\birdExpDesignedPassShare}{58.5\%}

\newcommand{\birdExpDesignedCorrectWithheld}{35.7\%}
\newcommand{\birdExpDesignedWrongWithheld}{43.7\%}
\newcommand{\birdExpDesignedEce}{0.155}
\newcommand{\birdExpDesignedBrier}{0.209}
\newcommand{\birdExpRuleAuroc}{0.694}
\newcommand{\birdExpRuleAurocCI}{0.666--0.723}
\newcommand{\birdExpRulePassAcc}{77.4\%}
\newcommand{\birdExpRulePassShare}{80.0\%}

\newcommand{\birdExpRuleCorrectWithheld}{14.0\%}
\newcommand{\birdExpRuleWrongWithheld}{25.4\%}

\newcommand{\birdExpAgreeFirstAuroc}{0.730}
\newcommand{\birdExpAgreeFirstAurocCI}{0.702--0.756}
\newcommand{\birdExpAgreeFirstPassAcc}{83.6\%}
\newcommand{\birdExpAgreeFirstPassShare}{62.0\%}

\newcommand{\birdExpAgreeFirstCorrectWithheld}{8.3\%}
\newcommand{\birdExpAgreeFirstWrongWithheld}{37.3\%}
\newcommand{\birdExpAgreeFirstEce}{0.098}
\newcommand{\birdExpAgreeFirstBrier}{0.175}
\newcommand{\birdExpAgreeFirstDelta}{+0.106}
\newcommand{\birdExpAgreeFirstDeltaCI}{+0.077 to +0.134}

\newcommand{\birdExpVthreeAuroc}{0.740}
\newcommand{\birdExpVthreeAurocCI}{0.713--0.768}
\newcommand{\birdExpVthreePassAcc}{84.6\%}
\newcommand{\birdExpVthreePassShare}{61.3\%}

\newcommand{\birdExpVthreeCorrectWithheld}{8.3\%}
\newcommand{\birdExpVthreeWrongWithheld}{37.3\%}

\newcommand{\birdExpAgreeAuroc}{0.721}
\newcommand{\birdExpAgreeAurocCI}{0.692--0.748}
\newcommand{\birdExpAgreePassAcc}{81.3\%}
\newcommand{\birdExpAgreePassShare}{78.9\%}

\newcommand{\birdExpAgreeCorrectWithheld}{6.3\%}
\newcommand{\birdExpAgreeWrongWithheld}{32.9\%}
\newcommand{\birdExpAgreeEce}{0.130}
\newcommand{\birdExpAgreeBrier}{0.183}
\newcommand{\birdExpFittedAuroc}{0.718}
\newcommand{\birdExpFittedAurocCI}{0.685--0.749}
\newcommand{\birdExpFittedPassAcc}{86.0\%}
\newcommand{\birdExpFittedPassShare}{53.3\%}

\newcommand{\birdExpFittedCorrectWithheld}{17.7\%}
\newcommand{\birdExpFittedWrongWithheld}{51.5\%}

\newcommand{\birdExpEmptyRight}{6.5\%}

\newcommand{\birdRlN}{1{,}441}

\newcommand{\birdRlAcc}{82.6\%}

\newcommand{\birdRlDesignedAuroc}{0.631}

\newcommand{\birdRlAgreeFirstAuroc}{0.756}

\newcommand{\birdRlAgreeFirstDelta}{+0.124}
\newcommand{\birdRlAgreeFirstDeltaCI}{+0.089 to +0.158}

\newcommand{\birdReN}{1{,}424}

\newcommand{\birdReAcc}{82.5\%}

\newcommand{\birdReExecMatch}{54.4\%}

\newcommand{\birdReDesignedAuroc}{0.589}

\newcommand{\birdReDesignedCorrectWithheld}{47.5\%}

\newcommand{\birdReDesignedEce}{0.157}
\newcommand{\birdReDesignedBrier}{0.166}

\newcommand{\birdReAgreeFirstAuroc}{0.755}

\newcommand{\birdReAgreeFirstCorrectWithheld}{11.4\%}

\newcommand{\birdReAgreeFirstEce}{0.093}
\newcommand{\birdReAgreeFirstBrier}{0.133}
\newcommand{\birdReAgreeFirstDelta}{+0.165}
\newcommand{\birdReAgreeFirstDeltaCI}{+0.129 to +0.204}

\newcommand{\birdReAgreeAuroc}{0.739}

\newcommand{\birdReAgreeEce}{0.097}
\newcommand{\birdReAgreeBrier}{0.134}

\newcommand{\birdReEmptyRight}{3.2\%}

\newcommand{\spiderRead}{1{,}034}
\newcommand{\spiderN}{983}
\newcommand{\spiderDropped}{51}

\newcommand{\spiderAcc}{92.6\%}
\newcommand{\spiderWrong}{73}
\newcommand{\spiderRefWrong}{127}

\newcommand{\spiderRefWrongMarkedCorrect}{120}

\newcommand{\spiderDesignedAuroc}{0.566}
\newcommand{\spiderDesignedAurocCI}{0.503--0.626}
\newcommand{\spiderDesignedPassAcc}{93.9\%}
\newcommand{\spiderDesignedPassShare}{49.7\%}

\newcommand{\spiderDesignedCorrectWithheld}{49.0\%}
\newcommand{\spiderDesignedWrongWithheld}{54.8\%}
\newcommand{\spiderDesignedEce}{0.187}
\newcommand{\spiderDesignedBrier}{0.125}
\newcommand{\spiderRuleAuroc}{0.682}
\newcommand{\spiderRuleAurocCI}{0.623--0.740}
\newcommand{\spiderRulePassAcc}{93.2\%}
\newcommand{\spiderRulePassShare}{85.1\%}

\newcommand{\spiderRuleCorrectWithheld}{13.7\%}
\newcommand{\spiderRuleWrongWithheld}{15.1\%}

\newcommand{\spiderAgreeFirstAuroc}{0.725}
\newcommand{\spiderAgreeFirstAurocCI}{0.663--0.787}
\newcommand{\spiderAgreeFirstPassAcc}{95.7\%}
\newcommand{\spiderAgreeFirstPassShare}{73.6\%}

\newcommand{\spiderAgreeFirstCorrectWithheld}{8.8\%}
\newcommand{\spiderAgreeFirstWrongWithheld}{35.6\%}
\newcommand{\spiderAgreeFirstEce}{0.088}
\newcommand{\spiderAgreeFirstBrier}{0.077}
\newcommand{\spiderAgreeFirstDelta}{+0.159}
\newcommand{\spiderAgreeFirstDeltaCI}{+0.097 to +0.221}
\newcommand{\spiderAgreeFirstSensAuroc}{0.721}

\newcommand{\spiderVthreeAuroc}{0.767}
\newcommand{\spiderVthreeAurocCI}{0.707--0.821}
\newcommand{\spiderVthreePassAcc}{96.8\%}
\newcommand{\spiderVthreePassShare}{70.7\%}

\newcommand{\spiderVthreeCorrectWithheld}{8.8\%}
\newcommand{\spiderVthreeWrongWithheld}{35.6\%}

\newcommand{\spiderAgreeAuroc}{0.749}
\newcommand{\spiderAgreeAurocCI}{0.688--0.809}
\newcommand{\spiderAgreePassAcc}{95.5\%}
\newcommand{\spiderAgreePassShare}{90.4\%}

\newcommand{\spiderAgreeCorrectWithheld}{3.7\%}
\newcommand{\spiderAgreeWrongWithheld}{35.6\%}
\newcommand{\spiderAgreeEce}{0.057}
\newcommand{\spiderAgreeBrier}{0.066}
\newcommand{\spiderFittedAuroc}{0.729}
\newcommand{\spiderFittedAurocCI}{0.651--0.804}
\newcommand{\spiderFittedPassAcc}{94.7\%}
\newcommand{\spiderFittedPassShare}{82.9\%}

\newcommand{\spiderFittedCorrectWithheld}{12.1\%}
\newcommand{\spiderFittedWrongWithheld}{32.9\%}

\newcommand{\spiderPassAccGapPts}{+1.8}
\newcommand{\spiderEmptyFires}{3.8\%}
\newcommand{\spiderEmptyFiredN}{37}
\newcommand{\spiderEmptyRight}{73.0\%}
\newcommand{\spiderEmptyGapPts}{19.6}

\newcommand{\lblN}{1{,}430}
\newcommand{\lblAgree}{90.8\%}
\newcommand{\lblKappa}{0.721}
\newcommand{\lblUp}{114}
\newcommand{\lblDown}{18}
\newcommand{\exAtCap}{534}
\newcommand{\exAtCapRight}{70.4\%}
\newcommand{\exBandNine}{84.4\%}
\newcommand{\exBandEight}{62.5\%}
\newcommand{\exBandSeven}{36.5\%}

\newcommand{\exEDesignedPassShare}{58.5\%}
\newcommand{\exEDropGroundedAuroc}{0.631}
\newcommand{\exEDropGroundedPassAcc}{82.5\%}
\newcommand{\exEDropGroundedPassShare}{48.2\%}
\newcommand{\exEDropSchemaAuroc}{0.625}
\newcommand{\exEDropSchemaPassAcc}{79.9\%}
\newcommand{\exEDropSchemaPassShare}{56.1\%}
\newcommand{\exEDropCompleteAuroc}{0.624}
\newcommand{\exEDropCompletePassAcc}{78.1\%}
\newcommand{\exEDropCompletePassShare}{58.4\%}
\newcommand{\exEDropSafetyAuroc}{0.555}
\newcommand{\exEDropSafetyPassAcc}{75.5\%}
\newcommand{\exEDropSafetyPassShare}{62.2\%}
\newcommand{\exESafetyNoAuroc}{0.572}
\newcommand{\exESafetyNoPassAcc}{75.7\%}
\newcommand{\exESafetyNoPassShare}{61.8\%}
\newcommand{\exESafetyMeanAuroc}{0.620}
\newcommand{\exESafetyMeanPassAcc}{75.8\%}
\newcommand{\exESafetyMeanPassShare}{61.8\%}
\newcommand{\exENoThinCapAuroc}{0.624}
\newcommand{\exENoThinCapPassAcc}{77.9\%}
\newcommand{\exENoThinCapPassShare}{58.5\%}
\newcommand{\exENoHardCapsAuroc}{0.638}
\newcommand{\exENoHardCapsPassAcc}{77.9\%}
\newcommand{\exENoHardCapsPassShare}{58.5\%}
\newcommand{\exEUnsupportedAnswerOnlyAuroc}{0.691}
\newcommand{\exEUnsupportedAnswerOnlyPassAcc}{77.9\%}
\newcommand{\exEUnsupportedAnswerOnlyPassShare}{78.4\%}
\newcommand{\exEUnsupportedIgnoredAuroc}{0.699}
\newcommand{\exEUnsupportedIgnoredPassAcc}{77.0\%}
\newcommand{\exEUnsupportedIgnoredPassShare}{83.8\%}

\newcommand{\exMaxMinorDelta}{0.007}

\newcommand{\exFireCapUnsupportedAnyN}{455}
\newcommand{\exFireCapUnsupportedAnyShare}{30.0\%}
\newcommand{\exFireCapUnsupportedAnyRight}{71.9\%}
\newcommand{\exFireUnsupportedCommentaryN}{360}
\newcommand{\exFireUnsupportedCommentaryShare}{23.8\%}
\newcommand{\exFireUnsupportedCommentaryRight}{75.6\%}

\newcommand{\exFireJudgeIntentMismatchRight}{32.5\%}

\newcommand{\exFireCapBadJoinRight}{46.8\%}

\newcommand{\exFireJudgeWrongFormulaRight}{44.7\%}
\newcommand{\exFireRareMax}{5.5\%}
\newcommand{\exPlanted}{480}
\newcommand{\exControlsOk}{480}
\newcommand{\exInjFakeTableWithheld}{100.0\%}
\newcommand{\exInjFakeTableNotPass}{100.0\%}
\newcommand{\exInjFakeTableOwnCheck}{100.0\%}
\newcommand{\exInjBadJoinWithheld}{100.0\%}
\newcommand{\exInjBadJoinNotPass}{100.0\%}
\newcommand{\exInjBadJoinOwnCheck}{100.0\%}
\newcommand{\exInjStandInWithheld}{100.0\%}
\newcommand{\exInjStandInNotPass}{100.0\%}
\newcommand{\exInjStandInOwnCheck}{100.0\%}
\newcommand{\exInjWrongNumberWithheld}{100.0\%}
\newcommand{\exInjWrongNumberNotPass}{100.0\%}
\newcommand{\exInjWrongNumberOwnCheck}{100.0\%}
\newcommand{\exInjTimeWindowWithheld}{90.0\%}
\newcommand{\exInjTimeWindowNotPass}{95.0\%}
\newcommand{\exInjTimeWindowOwnCheck}{83.3\%}
\newcommand{\exInjWrongAggregateWithheld}{75.0\%}
\newcommand{\exInjWrongAggregateNotPass}{81.7\%}
\newcommand{\exInjWrongAggregateOwnCheck}{96.7\%}
\newcommand{\exInjIntentFlipWithheld}{75.0\%}
\newcommand{\exInjIntentFlipNotPass}{91.7\%}
\newcommand{\exInjIntentFlipOwnCheck}{63.3\%}
\newcommand{\exInjPartialAnswerWithheld}{70.0\%}
\newcommand{\exInjPartialAnswerNotPass}{75.0\%}
\newcommand{\exInjPartialAnswerOwnCheck}{55.0\%}
\newcommand{\exInjJudgeRange}{70--90\%}

\newcommand{\abBirdExpDStrict}{+0.006}
\newcommand{\abBirdExpDStrictCI}{-0.002 to +0.014}
\newcommand{\abBirdExpDRelaxed}{+0.006}

\newcommand{\abBirdExpStL}{0.696}
\newcommand{\abBirdExpStSafety}{0.620}
\newcommand{\abBirdExpStWeighted}{0.638}
\newcommand{\abBirdExpStDesigned}{0.624}

\newcommand{\abBirdExpEmDz}{0.627}
\newcommand{\abBirdExpEmAg}{0.692}
\newcommand{\abBirdExpEmAf}{0.692}

\newcommand{\abBirdReDStrict}{+0.019}
\newcommand{\abBirdReDStrictCI}{+0.006 to +0.032}
\newcommand{\abBirdReDRelaxed}{+0.017}

\newcommand{\abBirdReDzPassRight}{593}
\newcommand{\abBirdReDzPassWrong}{98}
\newcommand{\abBirdReDzFlagRight}{24}
\newcommand{\abBirdReDzFlagWrong}{17}
\newcommand{\abBirdReDzAbstRight}{558}
\newcommand{\abBirdReDzAbstWrong}{134}
\newcommand{\abBirdReAgPassRight}{970}
\newcommand{\abBirdReAgPassWrong}{107}
\newcommand{\abBirdReAgFlagRight}{95}
\newcommand{\abBirdReAgFlagWrong}{35}
\newcommand{\abBirdReAgAbstRight}{110}
\newcommand{\abBirdReAgAbstWrong}{107}
\newcommand{\abBirdReAfPassRight}{752}
\newcommand{\abBirdReAfPassWrong}{61}
\newcommand{\abBirdReAfFlagRight}{289}
\newcommand{\abBirdReAfFlagWrong}{74}
\newcommand{\abBirdReAfAbstRight}{134}
\newcommand{\abBirdReAfAbstWrong}{114}

\newcommand{\abBirdReMovedAbstRight}{22}
\newcommand{\abBirdReMovedAbstWrong}{7}

\newcommand{\abBirdRePadAg}{+2.4}
\newcommand{\abBirdRePadAgCI}{+1.2 to +3.7}
\newcommand{\abBirdReStL}{0.715}
\newcommand{\abBirdReStSafety}{0.583}
\newcommand{\abBirdReStWeighted}{0.606}
\newcommand{\abBirdReStDesigned}{0.589}

\newcommand{\abBirdReEmDz}{0.574}
\newcommand{\abBirdReEmAg}{0.690}
\newcommand{\abBirdReEmAf}{0.699}

\newcommand{\abBirdReHeld}{264}

\newcommand{\abBirdReHeldWrongN}{46}

\newcommand{\abSpiderDStrict}{+0.011}
\newcommand{\abSpiderDStrictCI}{-0.001 to +0.023}
\newcommand{\abSpiderDRelaxed}{+0.006}

\newcommand{\abSpiderDzPassRight}{459}
\newcommand{\abSpiderDzPassWrong}{30}
\newcommand{\abSpiderDzFlagRight}{5}
\newcommand{\abSpiderDzFlagWrong}{3}
\newcommand{\abSpiderDzAbstRight}{446}
\newcommand{\abSpiderDzAbstWrong}{40}
\newcommand{\abSpiderAgPassRight}{849}
\newcommand{\abSpiderAgPassWrong}{40}
\newcommand{\abSpiderAgFlagRight}{27}
\newcommand{\abSpiderAgFlagWrong}{7}
\newcommand{\abSpiderAgAbstRight}{34}
\newcommand{\abSpiderAgAbstWrong}{26}
\newcommand{\abSpiderAfPassRight}{692}
\newcommand{\abSpiderAfPassWrong}{31}
\newcommand{\abSpiderAfFlagRight}{138}
\newcommand{\abSpiderAfFlagWrong}{16}
\newcommand{\abSpiderAfAbstRight}{80}
\newcommand{\abSpiderAfAbstWrong}{26}

\newcommand{\abSpiderMovedFlagWrong}{9}
\newcommand{\abSpiderMovedAbstRight}{46}

\newcommand{\abSpiderPadDzCI}{+0.3 to +3.4}
\newcommand{\abSpiderPadAg}{+0.2}
\newcommand{\abSpiderPadAgCI}{-0.5 to +1.0}
\newcommand{\abSpiderStL}{0.725}
\newcommand{\abSpiderStSafety}{0.565}
\newcommand{\abSpiderStWeighted}{0.584}
\newcommand{\abSpiderStDesigned}{0.566}
\newcommand{\abSpiderTiesAtCap}{437}
\newcommand{\abSpiderEmN}{1{,}034}
\newcommand{\abSpiderEmAcc}{68.9\%}
\newcommand{\abSpiderEmDz}{0.551}
\newcommand{\abSpiderEmAg}{0.651}
\newcommand{\abSpiderEmAf}{0.669}

\newcommand{\abSpiderEmKDz}{0.557}
\newcommand{\abSpiderEmKAg}{0.639}
\newcommand{\abSpiderEmKAf}{0.663}

\newcommand{\abSpiderHeld}{166}

\newcommand{\abSpiderHeldWrongN}{9}
\newcommand{\abSpiderHeldRightN}{157}

%% file: review_numbers.tex
\newcommand{\rvBirdExpJudgeAuroc}{0.578}
\newcommand{\rvBirdExpJudgeAurocCI}{0.552--0.604}

\newcommand{\rvBirdExpJudgeMinusAgree}{-0.142}
\newcommand{\rvBirdExpJudgeMinusAgreeCI}{-0.178 to -0.108}
\newcommand{\rvBirdExpAgreeJudgeAuroc}{0.735}

\newcommand{\rvBirdExpEntGreedy}{0.715}
\newcommand{\rvBirdExpEntComp}{0.714}
\newcommand{\rvBirdExpEntDiffN}{16}

\newcommand{\rvBirdReJudgeAuroc}{0.625}
\newcommand{\rvBirdReJudgeAurocCI}{0.591--0.659}

\newcommand{\rvBirdReJudgeMinusAgree}{-0.114}
\newcommand{\rvBirdReJudgeMinusAgreeCI}{-0.157 to -0.068}
\newcommand{\rvBirdReAgreeJudgeAuroc}{0.759}

\newcommand{\rvBirdReBrierConst}{0.1443}

\newcommand{\rvBirdReBrierAgreeFirst}{0.1327}

\newcommand{\rvBirdReEntGreedy}{0.721}
\newcommand{\rvBirdReEntComp}{0.723}
\newcommand{\rvBirdReEntDiffN}{13}

\newcommand{\rvBirdReSecMedian}{17.8}
\newcommand{\rvBirdReSecPninety}{32.5}
\newcommand{\rvSpiderJudgeAuroc}{0.625}
\newcommand{\rvSpiderJudgeAurocCI}{0.567--0.684}

\newcommand{\rvSpiderJudgeMinusAgree}{-0.122}
\newcommand{\rvSpiderJudgeMinusAgreeCI}{-0.186 to -0.061}
\newcommand{\rvSpiderAgreeJudgeAuroc}{0.758}

\newcommand{\rvSpiderBrierConst}{0.0687}

\newcommand{\rvSpiderBrierAgreeFirst}{0.0765}
\newcommand{\rvSpiderBrierAgree}{0.0662}

\newcommand{\rvSpiderEmAfMinusDz}{+0.106}
\newcommand{\rvSpiderEmAfMinusDzCI}{+0.069 to +0.143}
\newcommand{\rvSpiderEmAgMinusAf}{-0.024}
\newcommand{\rvSpiderEmAgMinusAfCI}{-0.050 to 0.000}
\newcommand{\rvSpiderEmVthreeMinusAg}{+0.031}
\newcommand{\rvSpiderEmVthreeMinusAgCI}{+0.006 to +0.059}
\newcommand{\rvSpiderAgMinusAf}{+0.025}
\newcommand{\rvSpiderAgMinusAfCI}{-0.008 to +0.052}
\newcommand{\rvSpiderVthreeMinusAg}{+0.017}
\newcommand{\rvSpiderVthreeMinusAgCI}{-0.022 to +0.060}
\newcommand{\rvSpiderEntGreedy}{0.743}
\newcommand{\rvSpiderEntComp}{0.744}
\newcommand{\rvSpiderEntDiffN}{1}

\newcommand{\rvSpiderSecMedian}{14.6}
\newcommand{\rvSpiderSecPninety}{25.6}
\newcommand{\rvTokCallsPerQ}{23.4}
\newcommand{\rvTokInPerQ}{171\,k}
\newcommand{\rvTokOutPerQ}{3.0\,k}
\newcommand{\rvTokJudgeInPerQ}{1.5\,k}
\newcommand{\rvTokJudgeOutPerQ}{1.6\,k}
\newcommand{\rvTokInPerSample}{34\,k}

%% file: review2_numbers.tex
\newcommand{\rtBirdExpClAfDz}{+0.105}
\newcommand{\rtBirdExpClAfDzCI}{+0.073 to +0.144}
\newcommand{\rtBirdExpClAgAf}{-0.008}
\newcommand{\rtBirdExpClAgAfCI}{-0.025 to +0.008}
\newcommand{\rtBirdExpClVtAg}{+0.018}
\newcommand{\rtBirdExpClVtAgCI}{+0.001 to +0.036}

\newcommand{\rtBirdReClAfDz}{+0.167}
\newcommand{\rtBirdReClAfDzCI}{+0.133 to +0.197}
\newcommand{\rtBirdReClAgAf}{-0.016}
\newcommand{\rtBirdReClAgAfCI}{-0.038 to +0.004}
\newcommand{\rtBirdReClVtAg}{+0.030}
\newcommand{\rtBirdReClVtAgCI}{+0.012 to +0.052}

\newcommand{\rtBirdReAfCW}{11.4\%}
\newcommand{\rtBirdReDzCWMatched}{40.2\%}

\newcommand{\rtSpiderClAfDz}{+0.161}
\newcommand{\rtSpiderClAfDzCI}{+0.098 to +0.234}
\newcommand{\rtSpiderClAgAf}{+0.025}
\newcommand{\rtSpiderClAgAfCI}{-0.022 to +0.066}
\newcommand{\rtSpiderClVtAg}{+0.017}
\newcommand{\rtSpiderClVtAgCI}{-0.042 to +0.089}

\newcommand{\rtSpiderAfWW}{35.6\%}
\newcommand{\rtSpiderAfCW}{8.8\%}
\newcommand{\rtSpiderDzCWMatched}{31.4\%}
\newcommand{\rtSpiderDzLoWW}{6.8\%}
\newcommand{\rtSpiderDzLoCW}{4.8\%}
\newcommand{\rtSpiderDzHiWW}{54.8\%}
\newcommand{\rtSpiderDzHiCW}{49.1\%}
\newcommand{\rtInjFakeTableCtrlWithheld}{35.0\%}

\newcommand{\rtInjFakeTableNewWithheldN}{39 of 39}

\newcommand{\rtInjBadJoinCtrlWithheld}{36.7\%}

\newcommand{\rtInjBadJoinNewWithheldN}{38 of 38}

\newcommand{\rtInjStandInCtrlWithheld}{40.0\%}

\newcommand{\rtInjStandInNewWithheldN}{36 of 36}

\newcommand{\rtInjWrongNumberCtrlWithheld}{40.0\%}

\newcommand{\rtInjWrongNumberNewWithheldN}{36 of 36}

\newcommand{\rtInjTimeWindowCtrlWithheld}{41.7\%}
\newcommand{\rtInjTimeWindowNewWithheld}{91.4\%}
\newcommand{\rtInjTimeWindowNewWithheldN}{32 of 35}

\newcommand{\rtInjWrongAggregateCtrlWithheld}{38.3\%}
\newcommand{\rtInjWrongAggregateNewWithheld}{70.3\%}
\newcommand{\rtInjWrongAggregateNewWithheldN}{26 of 37}
\newcommand{\rtInjWrongAggregateFiredButPass}{11}
\newcommand{\rtInjWrongAggregateFired}{58}
\newcommand{\rtInjIntentFlipCtrlWithheld}{33.3\%}

\newcommand{\rtInjIntentFlipNewWithheldN}{29 of 40}

\newcommand{\rtInjPartialAnswerCtrlWithheld}{33.3\%}

\newcommand{\rtInjPartialAnswerNewWithheldN}{28 of 40}

%% file: review3_numbers.tex
\newcommand{\rjBirdExpAuroc}{0.715}
\newcommand{\rjBirdExpAurocCI}{0.688--0.742}
\newcommand{\rjBirdExpMean}{0.771}
\newcommand{\rjBirdExpMinusAg}{-0.006}
\newcommand{\rjBirdExpMinusAgCI}{-0.042 to +0.030}
\newcommand{\rjBirdExpMeanMinusAg}{+0.050}
\newcommand{\rjBirdExpMeanMinusAgCI}{+0.031 to +0.070}

\newcommand{\rjBirdExpClMeanMinusAgCI}{+0.019 to +0.082}

\newcommand{\rjBirdExpEmMeanMinusAg}{+0.029}
\newcommand{\rjBirdExpEmMeanMinusAgCI}{+0.010 to +0.047}
\newcommand{\rjBirdReAuroc}{0.689}
\newcommand{\rjBirdReAurocCI}{0.653--0.723}
\newcommand{\rjBirdReMean}{0.769}
\newcommand{\rjBirdReMinusAg}{-0.049}
\newcommand{\rjBirdReMinusAgCI}{-0.094 to -0.007}
\newcommand{\rjBirdReMeanMinusAg}{+0.030}
\newcommand{\rjBirdReMeanMinusAgCI}{+0.009 to +0.052}

\newcommand{\rjBirdReClMeanMinusAgCI}{+0.012 to +0.050}

\newcommand{\rjBirdReEmMeanMinusAg}{+0.011}
\newcommand{\rjBirdReEmMeanMinusAgCI}{-0.007 to +0.029}
\newcommand{\rjSpiderAuroc}{0.805}
\newcommand{\rjSpiderAurocCI}{0.753--0.851}
\newcommand{\rjSpiderMean}{0.837}
\newcommand{\rjSpiderMinusAg}{+0.055}
\newcommand{\rjSpiderMinusAgCI}{-0.011 to +0.119}
\newcommand{\rjSpiderMeanMinusAg}{+0.088}
\newcommand{\rjSpiderMeanMinusAgCI}{+0.051 to +0.128}

\newcommand{\rjSpiderClMeanMinusAgCI}{+0.048 to +0.137}
\newcommand{\rjSpiderBrier}{0.0759}
\newcommand{\rjSpiderBrierConst}{0.0687}
\newcommand{\rjSpiderEmJudge}{0.605}
\newcommand{\rjSpiderEmAg}{0.639}

\newcommand{\rjSpiderEmMeanMinusAg}{+0.027}
\newcommand{\rjSpiderEmMeanMinusAgCI}{-0.002 to +0.056}

%% file: review4_numbers.tex
\newcommand{\rxN}{975}

\newcommand{\rxAgreeCross}{0.800}
\newcommand{\rxMeanJudgeCross}{0.846}
\newcommand{\rxAgreeCrossDeltaClCI}{+0.010 to +0.114}
\newcommand{\rxAgreeCrossDelta}{+0.052}
\newcommand{\rxTrioMinusMeanCI}{-0.009 to +0.025}

\newcommand{\rxEmAgreeCrossDelta}{+0.025}
\newcommand{\rxEmAgreeCrossDeltaCI}{+0.004 to +0.048}
\newcommand{\rxFires}{28.0\%}
\newcommand{\rxRightWhenFired}{81.0\%}
\newcommand{\rxVthreeAuroc}{0.786}
\newcommand{\rxVthreeAurocCI}{0.725--0.840}
\newcommand{\rxVthreeDelta}{+0.219}
\newcommand{\rxVthreeDeltaCI}{+0.138 to +0.291}
\newcommand{\rxVthreePassAcc}{97.6\%}
\newcommand{\rxVthreePassShare}{59.0\%}
\newcommand{\rxVthreeCorrectWithheld}{10.3\%}
\newcommand{\rxVthreeWrongWithheld}{46.6\%}
\newcommand{\rxAfSameAuroc}{0.727}
\newcommand{\rxAfSameWrongWithheld}{35.6\%}
\newcommand{\rxAfSameCorrectWithheld}{8.9\%}